\documentclass[sigconf,10pt, screen]{acmart}
\usepackage{times}  
\usepackage{hyperref}
\usepackage{titlesec}
\usepackage{todonotes}
\usepackage{tikz}
\usepackage{paralist}
\usepackage{subcaption}
\usepackage{afterpage}
\usetikzlibrary{quotes}
\usepackage{algorithm}
\usepackage{algorithmicx}
\usepackage{algpseudocode} 
\usepackage{xcolor}
\usepackage{pgfplots}
\usepackage[flushmargin]{footmisc}
\pgfplotsset{compat=1.18} 
\usetikzlibrary{automata}
\usepackage{tikz}
\usetikzlibrary{shapes.misc, shapes.geometric, shapes.multipart, shapes.arrows, positioning, quotes, arrows.meta, decorations.pathmorphing, matrix, graphs, fit, calc, automata, positioning, shadows}

\tikzset{
	state/.style={
		rectangle,
		rounded corners,
		draw=black, thick,
		minimum height=2em,
		inner sep=2pt,
		text centered,
	},
}

\algnewcommand{\IIf}[1]{\State\algorithmicif\ #1\ \algorithmicthen}
\algnewcommand{\EndIIf}{\unskip\ \algorithmicend\ \algorithmicif}
\algblock{Input}{EndInput}
\algnotext{EndInput}
\algblock{Output}{EndOutput}
\algnotext{EndOutput}

\algblock{OuterVar}{EndOuterVar}
\algnotext{EndOuterVar}

\hypersetup{pdfstartview=FitH,pdfpagelayout=SinglePage}

\usepackage[commandnameprefix=ifneeded]{changes}
\definechangesauthor[name={Pantea}, color=orange]{PK}

\usepackage{dsfont}
\usepackage[T1]{fontenc}
\usepackage[utf8]{inputenc}
\usepackage{upquote}
\usepackage{microtype}
\UseMicrotypeSet[protrusion]{basicmath}
\usepackage{graphicx,grffile}

\usepackage{endnotes,epsfig}
\usepackage{multirow}
\usepackage{xspace}
\usepackage{tabularx}
\usepackage{ragged2e}
\usepackage{booktabs}
\usepackage{paralist}
\usepackage[american]{babel}
\usepackage[shortlabels]{enumitem}
\usepackage{color}
\usepackage{wrapfig}
\usepackage{balance}
\usepackage{epstopdf}
\usepackage{url}
\usepackage{pifont}

\usepackage{mathtools}
\usepackage[normalem]{ulem}
\usepackage{xkeyval}

\usepackage{ifthen}
\usepackage{etoolbox}

\usepackage{sepfootnotes}
\usepackage{bbding}

\newcommand\paraspace{\vspace*{0.5ex}}
\providecommand\parab[1]{\paraspace\noindent\textbf{#1}}

\newcommand{\sysname}{$\mathcal{X}$Plain\xspace}

\definecolor{mygreen}{RGB}{208, 244, 222}
\definecolor{myp}{RGB}{253, 226, 228}
\definecolor{myb}{RGB}{198, 222, 241}
\definecolor{mymagenta}{RGB}{255,0,255} 
\definecolor{myo}{RGB}{255,200,0}
\definecolor{pastelblue}{RGB}{102,153,204}
\definecolor{forestgreen}{rgb}{0.13, 0.65, 0.13}

\newif\ifshowproof
\showprooftrue 

\makeatletter
\renewcommand{\subsectionautorefname}{\S\@gobble}
\renewcommand{\sectionautorefname}{\S\@gobble}
\renewcommand{\subsubsectionautorefname}{\S\@gobble}
 
\makeatother

\title{Towards Safer Heuristics With $\mathcal{X}$Plain}

\author{Pantea Karimi\textsuperscript{1*}, Solal Pirelli\textsuperscript{2*}, Siva Kesava Reddy Kakarla\textsuperscript{3}, Ryan Beckett\textsuperscript{3}, Santiago Segarra\textsuperscript{4}, Beibin Li\textsuperscript{3}, Pooria Namyar\textsuperscript{5}, Behnaz Arzani\textsuperscript{3}}

\affiliation{%
  \institution{\textsuperscript{1}MIT \quad \textsuperscript{2}EPFL, Sonar \quad \textsuperscript{3}Microsoft Research \quad
  \textsuperscript{4}Rice University \quad
  \textsuperscript{5}University of Southern California}
  \country{}
}

\ccsdesc[500]{Networks~Network performance analysis}
\ccsdesc[500]{Networks~Network performance modeling}
\ccsdesc[500]{Networks~Network management}
\ccsdesc[500]{Networks~Network reliability}

\begin{document}

\renewcommand{\shortauthors}{P. Karimi \textit{et al.}}
\acmYear{2024}\copyrightyear{2024}
\acmConference[HOTNETS '24]{The 23rd ACM Workshop on Hot Topics in Networks}{November 18--19, 2024}{Irvine, CA, USA}
\acmBooktitle{The 23rd ACM Workshop on Hot Topics in Networks (HOTNETS '24), November 18--19, 2024, Irvine, CA, USA}
\acmDOI{10.1145/3696348.3696884}
\acmISBN{979-8-4007-1272-2/24/11}

\begin{abstract}

Many problems that cloud operators solve are computationally expensive, and operators often use heuristic algorithms (that are faster and scale better than optimal) to solve them more efficiently.
Heuristic analyzers enable operators to find when and by how much their heuristics underperform. However, these tools do not provide enough detail for operators to mitigate the heuristic's impact in practice: they only discover a \emph{single input instance} that causes the heuristic to underperform (and not the full set) and they do not \emph{explain why}.


We propose $\mathcal{X}$Plain, a tool that extends these analyzers and helps operators understand when and why their heuristics underperform. 
We present promising initial results that show such an extension is viable. 
\end{abstract}

\keywords{Heuristic Analysis, Explainable Analysis, Domain-Specific Language}
\maketitle


\section{Introduction}

Operators use heuristics (approximate algorithms that are faster or scale better than their optimal counterparts) in production systems to solve computationally difficult or expensive problems. 
These heuristics perform well across many typical instances, but they can break in unexpected ways when network conditions change~\cite{metaopt,venkatarxive,venkatstarvation,mina}.
Our community has developed tools that enable operators to identify such situations~\cite{metaopt,venkat1,venkatarxive,venkatstarvation,venkatNSDI24}. 
These tools find the ``performance gap'' of one heuristic algorithm compared to another heuristic or the optimal~---~they identify an example instance of an input which causes a given heuristic to underperform.

For example, MetaOpt~\cite{metaopt} describes a heuristic deployed in Microsoft's wide area traffic engineering solution and shows it could underperform by $30\%$  (see~\autoref{sec::heuristicanalysis}).
This means the company would either have to overprovision their networks to support $30\%$ more traffic, drop that traffic, or delay it. 

The potential benefit of heuristic analyzers is clear: they allow operators to quantify the risk of heuristics they want to deploy. 
Although these heuristic analyzers have already shed light on the performance gap of many deployed heuristics, they are still in their nascent stage and have limited use for operators who do not have sufficient expertise in formal methods and/or optimization theory. 
There are crucial features missing: operators have to (1) model the heuristics they want to analyze in terms of mathematical constructs these tools can support and (2) manually analyze the outputs from these tools to understand \emph{how} to fix their heuristics or their scenarios~---~the tool only provides a performance gap and an example input that caused it. 
They do not produce the full space of inputs that can cause large gaps nor describe why the heuristic underperformed in these instances.

The latter problem limits the operator's ability to use the output of these tools to fix the problem and to either improve the heuristic, create an alternative solution for when it underperforms, or cache the optimal solution for those instances. 
In our earlier examples, the operator has to look at the tool's example demand matrix to understand why the heuristic routes $30\%$ less traffic than the optimal. 

The state of these heuristic analyzers today is reminiscent of the early days of our community's exploration of network verifiers and their potential to help network operators configure and manage their networks. 
In the same way that network verifiers enabled operators to identify bugs in their configurations~\cite{atomic, hsa, veriflow, apkeep, nod, anteater, deltanet, fibverifier, minesweeper, era, arc, campion, groot}, a heuristic analyzer can help them find the performance gap of the algorithms they deploy. 
Tools that allow operators to leverage heuristic analyzers more easily, identify \emph{why} the heuristics underperform, and devise solutions to remediate the issue serve a similar purpose to the tools our community crafted that \emph{explained} the impact of configuration bugs~\cite{campion, selfstarter, secguru, tian19safely} (by producing all sets of packets that the bug impacted and the configuration lines that caused the impact).

We propose $\mathcal{X}$Plain~---~our vision for a ``generalizer'' that can augment existing heuristic analyzers and help operators either improve their heuristics (by helping them find \emph{why} the heuristics underperform) or use them more safely (by finding all regions where they underperform).

We propose a domain-specific language (\autoref{sec::dsl}), which allows us to concretely describe the heuristic's behavior and that of a benchmark we want to compare it to for automated analysis. It is rooted in network flow abstraction, which allows us to model the behavior of many heuristics that operators use in today's networks, including \emph{all} those from~\cite{metaopt,venkatarxive}. Our compiler converts inputs in this language into an existing heuristic analyzer. Our \emph{efficient} iterative algorithm uses the analyzer, extrapolates from the adversarial inputs it finds, and finds all adversarial subspaces where the heuristic underperforms. We then use our language again and visualize \emph{why} (i.e., the different decisions the heuristic made compared to the optimal that caused it to underperform) the heuristic underperforms in these cases.

We also discuss open questions and a possible approach built on the solutions we propose in this work to uncover what \emph{properties} in the input or the problem instance cause the heuristic to underperform (\autoref{sec:generalizer}).

Our proof-of-concept implementation of this idea uses MetaOpt~\cite{metaopt} as the underlying heuristic analyzer because it is open source. But our proposal applies to other heuristic analyzers such as~\cite{venkat1, venkatarxive, venkatNSDI24} as well. 


\begin{figure*}
 \centering  
    \begin{subfigure}{0.4\textwidth}  
        \centering  
	\begin{tikzpicture}[  
	scale=0.6, transform shape,  
	node distance=0.75cm,  
	font=\small,  
	every edge/.style={draw, -Stealth},  
	every edge quotes/.style={auto, font=\footnotesize},  
	state/.style={circle, draw}  
	]  
	\begin{scope}  
	\node[state, fill=myb] (1) {1};  
	\node[state, fill=myb, right=of 1] (2) {2};  
	\node[state, fill=myb, right=of 2] (3) {3};  
	\node[state, fill=myb, below=of 1] (4) {4}; 
	\node[state, fill=myb, right=of 4] (5) {5};  
	
	\path  
	(1) edge[above, "100"] (2)  
	(2) edge[above, "100"] (3)  
	(4) edge[below, "50"] (5)  
	(1) edge[below, "50"] (4)  
	(5) edge[below, "50"] (3)  
	;  
	\end{scope}  
	
	\node[anchor=west] at ([yshift=-0.5cm]3.east){  
	\begin{tabular}{cc | cc | cc}  
	\multicolumn{2}{c}{Demand} & \multicolumn{2}{c}{DP (thresh = 50)} & \multicolumn{2}{c}{OPT} \\  
	\hline  
	src-dest & value & path & value & path & value \\  
	\hline  
        $1$$\rightsquigarrow$$3$ & 50 & $1$-$2$-$3$ & 50 & $1$-$4$-$5$-$3$ & 50\\
        $1$$\rightsquigarrow$$2$ & 100 & $1$-$2$ & 50 & $1$-$2$ & 100\\
        $2$$\rightsquigarrow$$3$ & 100 & $2$-$3$ & 50 & $2$-$3$ & 100\\
	\hline
	& & Total DP & \textbf{150} & Total OPT & \textbf{250}  \\
	\cline{3-6}
	\end{tabular}  } ; 
	\end{tikzpicture}  
\caption{\small DP from~\cite{metaopt}. (left) Topology. (right) A set of demands and their flow allocations using the DP heuristic and the optimal (OPT) solution. \label{f:pinning_issues}}
    \end{subfigure} 
    \hspace{10pt}
    \begin{subfigure}{0.28\textwidth}  
        \centering  
        \begin{tikzpicture}[scale=0.8, transform shape,  
	node distance=0.75cm,  
	font=\small]
        \node[ state, text width=6cm,
				node distance=7.0cm,
				anchor=center] (B)
				{%
					\begin{algorithmic}[t]
						\OuterVar : $d_k$ requested rate of demand $k$ \EndOuterVar
						\Input : $P_k$ paths for demand $k$ \EndInput
						\Input : $\hat{p}_k$ shortest path \EndInput
						\Input : $T_d$ demand pinning threshold \EndInput
						\ForAll{{$\text{demand}~k \in \mathcal{D}$}}
						\State ${\sf ForceToZeroIfLeq}(d_k - f^{\hat{p}_k}_k, d_k, T_d)$
						\EndFor
						\State {\sf MaxFlow()}
					\end{algorithmic}
				};
        \end{tikzpicture}  
        \caption{\small DP in MetaOpt.}  
        \label{fig:metaoptEncoding}  
    \end{subfigure}
    \begin{subfigure}{0.28\textwidth}
    \centering  
    \begin{tikzpicture}[scale=0.7, transform shape,
	font=\small]
\node[state, text width=6cm]
{%
	\begin{algorithmic}[t]  
	\Statex \textbf{OuterVar}: ${\bf Y}$(size of balls)  
	\Statex \textbf{Input}: ${\bf C}$(capacity of bins)  
	\ForAll{$\text{ball}~~i \textbf{ and } \text{bin}~~j$}  
	\begin{align*}  
	\hspace{0.5cm}&{\bf r}_{ij} = {\bf C}_j - {\bf Y_{i}} - \sum\limits_{\text{ball}~u < i}{{\bf x}_{uj}^d}\\  
	&f_{ij} = {\sf AllLeq([-r^d_{ij}]_d, 0)} \\  
	&\gamma_{ij} = {\sf AllEq([x^d_{ik}]_{d, k<j}, 0)} \\  
	&\alpha_{ij} = {\sf AND}(f_{ij}, \gamma_{ij}) \\  
	\hspace{0.5cm}&{\sf IfThenElse}({\sf \alpha_{ij}}, [({\bf x}_{ij}, {\bf Y}_i)], [({\bf x}_{ij}, 0)])  
	\end{align*}  
	\EndFor  
	\end{algorithmic}  
};  
\end{tikzpicture}   
		\caption{\small Heuristic for VBP in MetaOpt.}  
		\label{fig:metaoptVBP}
    \end{subfigure}
    \caption{Example heuristics and their encoding in MetaOpt (sub-figures (b) and (c)). Heuristic in sub-figure (b) forces the demands less than a threshold to be pinned and then solves a flow maximization problem, heuristic in sub-figure (c) assigns the first bin that can fit the ball.\label{fig:DP}}
\end{figure*}

\definecolor{ballcolor}{rgb}{0.71, 0.11, 0.5}
\begin{figure}[h!]
    \centering
    \begin{tikzpicture}[scale=0.6]

    \def\balls{{0.3, 0.8, 0.2, 0.4, 0.7, 0.7, 0.15, 0.85, 0.25, 0.25, 0.3, 0.75, 0.75, 0.6, 0.12, 0.4, 0.4}}
    \def\spacing{0.77} 

    \foreach \i in {0,...,13, 14, 15, 16} {
        \pgfmathsetmacro{\size}{\balls[\i]}
        \if \i == 11
            \pgfmathsetmacro{\pos}{\i * \spacing - 0.1}
        \else
            \pgfmathsetmacro{\pos}{\i * \spacing}
        \fi
        \shade[ball color=ballcolor] (\pos,-0.75) circle (\size/2); 
        \ifodd\i
            \node at (\pos,-1.4) {\scriptsize \size}; 
        \else
            \node at (\pos,-0.2) {\scriptsize \size};  
        \fi
    }

    \end{tikzpicture}
    \caption{Example adversarial instance for FF with equal-sized bins with size of 1; the optimal uses 8 bins and the heuristic 9.
    }
    \vspace{-20 pt}
    \label{fig:binpacking}
\end{figure}
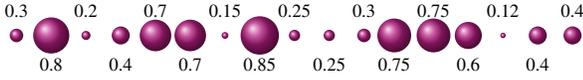

\section{What is heuristic analysis?}
\label{sec::heuristicanalysis}
Heuristic analyzers~\cite{metaopt, venkat1, venkatarxive} take a \emph{heuristic model} and a \emph{benchmark model} (e.g., the optimal) as input. Their goal is to characterize the performance gap of the heuristic compared to the benchmark. 
Recent tools~\cite{metaopt,venkatarxive} use optimization theory or first-order logic to solve this problem and return a single input instance that causes the heuristic to underperform.




\noindent \textbf{Example heuristics}  from these work include:

\noindent \textit{Demand Pinning (DP)} was deployed in Microsoft's wide area network. DP is a heuristic for the traffic engineering problem. The optimal algorithm assigns traffic (demands) to paths and maximizes the total flow it routes through the network without exceeding the network capacity. 
Operators use DP to reduce the size of the optimization problem they solve. 
DP first filters all demands below a pre-defined threshold and routes them through (pins them to) their shortest path. 
It then routes the remaining demands optimally using the available capacity (see~\autoref{fig:DP}).

MetaOpt authors modeled DP directly as an optimization problem. They also provided a number of helper functions that allow operators to model it more easily (\autoref{fig:metaoptEncoding}).
MetaOpt solves a bi-level optimization that produces the performance gap and demand that causes it (the flow in~\autoref{f:pinning_issues}). 
It is easy to see what is missing: it is up to the operator to examine the single output and find why DP underperformed. 
DP is amenable to such manual analysis (see~\cite{metaopt}), but not all heuristics are. It is also hard for operators to extrapolate from this example adversarial input and find all other regions of the input space where DP may underperform.
These limitations are exacerbated as we move to larger problems with more demands, where it is harder to pinpoint how a heuristic's decision to route a particular demand interferes with its ability to route others.

\textit{Vector bin packing (VBP)} places multi-dimensional balls into multi-dimensional bins and minimizes the number of bins in use. 
Operators use VBP in many production systems, such as to place VMs onto servers~\cite{vmlifetime}. 

The VBP problem is APX-hard~\cite{vbpAPX}. 
One heuristic that solves VBP is first-fit (FF), which greedily places an incoming ball in the first bin it fits in. \autoref{fig:metaoptVBP} shows how we can encode it in MetaOpt. 

MetaOpt produces the adversarial ball sizes 1\%, 49\%, 51\%, 51\% (as a percentage of the bin size) for an example with 4 balls and 3 equal-sized bins (we use single-dimensional balls)~---~the optimal uses 2 bins while FF uses 3 (we show a more complex version in~\autoref{fig:binpacking}). Once again, operators have to reason through this example to identify \emph{why} FF underperforms and what \emph{other inputs} cause the same problem. 
This is harder in FF and other VBP heuristics, such as best fit or first fit decreasing, as evidenced by the years of research by theoreticians in this space~\cite{theoryVBP}.

In this paper, we use the DP and VBP as running examples. These examples are representative of the heuristics prior work has studied~\cite{metaopt,venkatarxive} (the scheduling examples Virley studies are conceptually similar to VBP, and we think our discussions directly translate to those use-cases).

Prior work~\cite{mina} shows that, using a single adversarial instance, it is difficult to understand why a heuristic underperformed. 
It is even harder to generalize from why an adversarial input causes the heuristic to underperform on a single problem instance (or a few instances) to what \emph{properties} in the input and the problem instance cause it to underperform. 

\section{The case for comprehensive analysis}
\label{sec::case}

Prior work~\cite{metaopt,venkatNSDI24, mina} show explaining adversarial inputs can have benefits: we can improve DP's performance gap by an order of magnitude and produce congestion control algorithms that meet pre-specified requirements~\cite{venkatNSDI24}. But these results require manual analysis~\cite{metaopt} or problem-specific models~\cite{venkatNSDI24, mina}. 

We see an opportunity for a new tool that enables operators to identify the full \emph{risk surface} of the heuristic (the set of inputs where the heuristic underperforms) and to identify \emph{why} the heuristic underperforms automatically. 
It can produce (1) a description of the entire area(s) where a heuristic has a high performance gap; or (2) a description of what choices the heuristic makes that cause it to underperform (the difference in the actions of the heuristic and the optimal can point us to \emph{why} the heuristic underperforms). 
Through these outputs, these tools can make it safer for operators to use heuristics in practice as they can mitigate the cases where they underperform and maybe even design safer heuristics.

There are three levels of information we can provide: (1) for \emph{a given problem instance}, the \emph{sets of inputs} that cause the heuristic to underperform; (2) for \emph{a given problem instance}, a reason as to \emph{why} the heuristic underperforms in each contiguous region of the adversarial input space; and (3) for \emph{the general case}, the \emph{characteristics} of the inputs and problem instances that cause the heuristic to underperform.

Take DP as an example. The ideal tool would produce:

\noindent \textbf{Type 1.} For a given topology, the adversarial input sets are of the form $\cup D_i$ where each $D_i \in \mathbb{R}^n_+$ represents a contiguous subspace of the n-dimensional (8-dimensional in ~\autoref{f:pinning_issues} for 8 demands) space.

For a given $D_i$: (a) an entry $d_{ij} = T - \epsilon$ (here $T$ is the demand pinning threshold and $\epsilon$ is a small positive value) if there are multiple paths between the nodes $i$ and $j$ (we call a demand $d: d \le T$ a pinnable demand); (b) for all other $uv$ where a portion of the path between the nodes $u$ and $v$ intersects with the shortest path of a pinnable demand we have $d_{uv} \ge \min({\mathcal{C}_{uv}} - T)$. Here, the set $\mathcal{C}_{uv}$ contains the capacity of all links on the path between $u$ and $v$. The adversarial instance in our example in~\autoref{f:pinning_issues} fits this behavior.

\noindent \textbf{Type 2.} For a given topology, DP routes pinned demands on their shortest paths, but the optimal routes them through alternate paths. We expect the pinned demands in each contiguous subspace would all have a common pattern where they have the same shortest path, and DP does shortest-path routing for these demands, whereas the optimal does not. 

\noindent \textbf{Type 3.} The heuristic's performance is worse when the length of the shortest path of the pinned demands is longer or the capacity of the links along these paths is lower~---~pinned demands limit the heuristic's ability to route other demands.

\section{Challenges}
\label{sec:challenges}

It is hard to arrive at low-level models of a heuristic in order to use existing analyzers~\cite{metaopt,venkatarxive,venkatNSDI24}, and operators need to have expertise in either formal methods~\cite{venkatarxive,venkatNSDI24} or optimization theory~\cite{metaopt, boyd2004convex} to do so. 
We see an analogy with writing imperative programs in assembly code: we can write any program in assembly but it takes time, has a high risk of being buggy, and makes code reviews (i.e., explanations) difficult.

Low-level models operate over variables and constructs that are often hard to connect to the original problem (``Greek letters'' and ``auxiliary variables'' instead of ``human-readable'' text). 
To model the first fit behavior, MetaOpt uses an auxiliary, binary variable $\alpha_{ij}$ that captures whether bin $j$ is the first bin where ball $i$ fits in, and sets its value through:

{\small 
	\begin{align*}
	&\alpha_{ij} \leq \frac{{f}_{ij} + \sum_{\{k \in \textsc{\textsf{bins}} \mid k<j\}}(1- {f}_{ik})}{j} \quad \forall i \in \textsc{\textsf{balls}},\ \forall j \in \textsc{\textsf{bins}} \\
 	&\sum\limits_{j \in \textsc{\textsf{bins}}}\alpha_{ij}  = 1 \quad \forall i \in \textsc{\textsf{balls}}.
	\end{align*}
}

It is hard to derive an explanation from such a model and harder still to connect it to how the heuristic works to explain its behavior.
We need a better and more descriptive language to encode the behavior of the heuristic. We also need to:

\noindent \textbf{Find adversarial subspaces and validate them.} 
These are subspaces of the input space where the inputs that fall in those subspaces cause the heuristic to underperform. 
To find them, we need a search algorithm that iterates and extrapolates from the adversarial inputs existing analyzers find (similar to the all-SAT problem~\cite{allsat1, allsat2, allsat3}, the input space is large, and we cannot blindly search it to find adversarial inputs~\cite{metaopt}).
Once we find a potential "adversarial subspace," we should validate it: we need to check whether the heuristic's performance gap is higher for inputs that belong to the adversarial subspace compared to those that do not with statistical significance. 

\noindent \textbf{Find why the inputs in each subspace cause bad performance.} It is reasonable to assume the inputs in the same contiguous adversarial subspace trigger the same ``bad behavior'' in the heuristic. 
To find and explain these behaviors, we need to automatically reason through the heuristic's actions and compare them to those of the benchmark: we need to \emph{concretely} encode the heuristic and benchmark's choices as part of the language we design for our solution. 
The challenge is to ensure this language applies to a broad range of problems and is amenable to the types of automation we desire.

\noindent \textbf{Generalize beyond a single instance.} 
Perhaps the hardest challenge is to generalize from the instance-based explanations to one that applies to the heuristic's behavior in the general case: we have to find a valid extrapolation from these instance-based examples and discover patterns that apply to the heuristic's behavior across different problem instances.

\section{The $\mathcal{X}$Plain proposal}

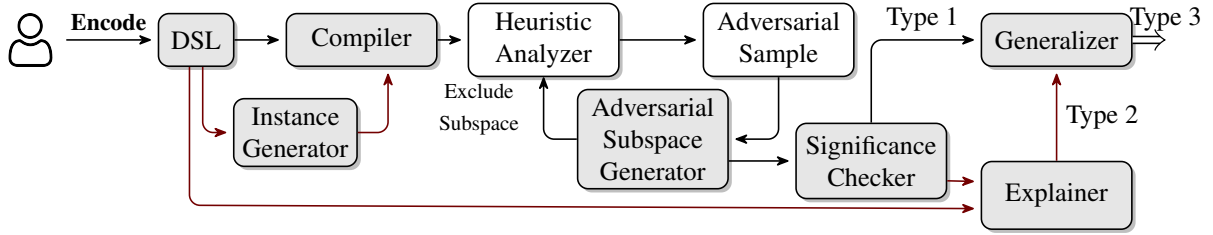
\begin{figure*}[!t]
    \centering
\begin{tikzpicture}[every arrow/.style={}, scale = 0.88]
    \node[] (user) at (0,1.5) { \includegraphics[width=.7cm]{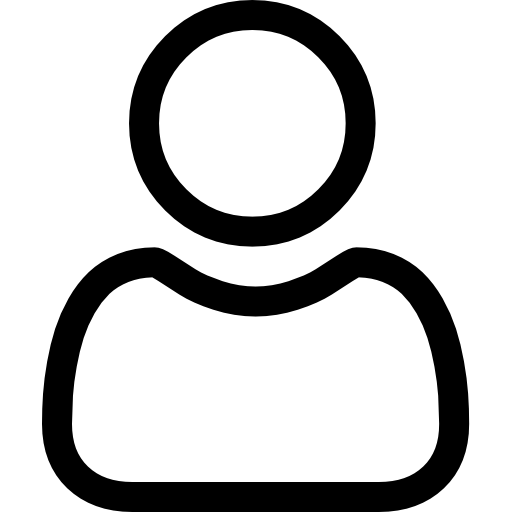} };
    \node[draw, minimum width=1cm, minimum height=.7cm, rounded corners, fill=gray!20, semithick, text=black, drop shadow={shadow xshift=0.5mm, shadow yshift=-0.5mm, opacity=0.5}] (dsl) at (2.5, 1.5) { DSL };
    \node[draw, minimum width=1cm, minimum height=.9cm, align=center, rounded corners, fill=gray!20, semithick, text=black, drop shadow={shadow xshift=0.5mm, shadow yshift=-0.5mm, opacity=0.5}] (instance) at (4, 0.1) {Instance \\ Generator};
    \node[draw, minimum width=2cm, minimum height=.7cm, rounded corners, fill=gray!20, semithick, text=black, drop shadow={shadow xshift=0.5mm, shadow yshift=-0.5mm, opacity=0.5}] (compiler) at (5, 1.5) { Compiler };
    %
    \node[draw, minimum width=2cm, minimum height=.9cm, rounded corners, align=center, fill=white, semithick, text=black, drop shadow={shadow xshift=0.5mm, shadow yshift=-0.5mm, opacity=0.5}] (metaopt) at (7.75, 1.5) { Heuristic\\ Analyzer};
    \node[draw, minimum width=2cm, minimum height=.9cm, align=center, rounded corners, fill=white, semithick, text=black, drop shadow={shadow xshift=0.5mm, shadow yshift=-0.5mm, opacity=0.5}] (sample) at (11.3, 1.5) { Adversarial\\ Sample};
    \node[draw, minimum width=2cm, minimum height=.11cm, align=center, rounded corners, fill=gray!20, semithick, text=black, drop shadow={shadow xshift=0.5mm, shadow yshift=-0.5mm, opacity=0.5}] (space) at (9.4, 0) { Adversarial\\ Subspace\\ Generator};
    \node[draw, minimum width=2cm, minimum height=1cm, align=center, rounded corners, fill=gray!20, semithick, text=black, drop shadow={shadow xshift=0.5mm, shadow yshift=-0.5mm, opacity=0.5}] (significance) at (12.7, -0.33) { Significance\\ Checker};
    \node[draw, minimum width=2cm, minimum height=.9cm, align=center, rounded corners, fill=gray!20, semithick, text=black, drop shadow={shadow xshift=0.5mm, shadow yshift=-0.5mm, opacity=0.5}] (explainer) at (15.5, -0.85) {Explainer};
    \node[draw, minimum width=2cm, minimum height=.7cm, align=center, rounded corners, fill=gray!20, semithick, text=black, drop shadow={shadow xshift=0.5mm, shadow yshift=-0.5mm, opacity=0.5}] (generalizer) at (15.5, 1.5) {Generalizer};

    \draw[->, semithick, draw=black, rounded corners, >={Stealth[round]}] 
    (user.east) -- ($(dsl.west)+(-0.1,0)$) 
    node[midway, above, sloped, text=black, font=\small\bfseries] {Encode};

    \draw[->, semithick, draw=black, rounded corners, >={Stealth[round]}, ] (dsl.east) -- ($(compiler.west)+(-0.1,0)$);
    \draw[->, semithick, draw=black, rounded corners, >={Stealth[round]}] (compiler.east) -- ($(metaopt.west)+(-0.1,0)$);
    \draw[->, semithick, draw=black, rounded corners, >={Stealth[round]}] (metaopt.east) -- ($(sample.west)+(-0.1,0)$);
    \draw[->, semithick, draw=black, rounded corners, >={Stealth[round]}] ($(space.west)+(0,0)$) -- ($(metaopt.south |- space.west) +(0,0)$) -- ($(metaopt.south)$);

    \node[align=center, xshift=-1.3cm, yshift=0.4cm] (exclude) at (space.west) {\footnotesize Exclude\\ \footnotesize Subspace};

    \draw[->, semithick, draw=black, rounded corners, >={Stealth[round]}] (sample.south) -- ($(sample.south)!(space.east)!(sample.south |- space.east) +(0,0)$) -- ($(space.east)+(0.1,0)$);
    
    \draw[->, semithick, draw=red!45!black, rounded corners, >={Stealth[round]}] ($(instance.east)$) -- ($(compiler.south |- instance.east) +(0.4,0)$) -- ($(compiler.south)+(0.4, -0.1)$);
    \draw[->, semithick, draw=red!45!black, rounded corners, >={Stealth[round]}] ($(dsl.south) - (-0.1, 0)$) -- ($(dsl.south |- instance.west) - (-0.1, 0)$) -- ($(instance.west)+(-0.1, 0)$);
    \draw[->, semithick, draw=red!45!black, rounded corners, >={Stealth[round]}] ($(dsl.south)+(-0.1,0)$) -- ($(dsl.south |- instance.west) +(-0.1,-1.1)$) -- ($(explainer.west)+(-0.1, -0.2)$);
    
    \draw[->, semithick, draw=black, rounded corners, >={Stealth[round]}] ($(significance.east)+(-3.3,0)$) -- ($(significance.west)+(-0.1,0)$);
    
    \draw[->, semithick, draw=red!45!black, rounded corners, >={Stealth[round]}]  ($(significance.east)+(0,-0.3)$) -- ($(explainer.west)+(-0.1,0.21)$);
    \draw[->, semithick, draw=red!45!black, rounded corners, >={Stealth[round]}]  ($(explainer.north)+(0,0)$) -- ($(generalizer.south)+(0,-0.1)$) node[midway, right] {Type 2};
    \draw[->, semithick, draw=black, rounded corners, >={Stealth[round]}] ($(significance.north)$) -- ($(significance.north |- generalizer.west)$) -- ($(generalizer.west)+ (-0.1,0)$)  node[midway, above] {Type 1};
    \draw[double, double distance=2pt, semithick, -{Classical TikZ Rightarrow[length=1.5mm]}]  ($(generalizer.east)$) -- ($(generalizer.east)+(0.5,0)$) node[at end, above] {Type 3};

\end{tikzpicture}
    \caption{$\mathcal{X}$Plain: the system architecture we propose to extend existing heuristic analyzers.} 
    
    \label{fig::proposal}
\end{figure*}

\newcommand{\sink}{\raisebox{2.3pt}{\protect\tikz[baseline=(char.base)]{\protect\node[shape=rectangle,draw, fill=forestgreen!20, minimum size=1mm] (char){};}}}
\newcommand{\source}{\raisebox{2.3pt}{\protect\tikz[baseline=(char.base)]{\protect\node[shape=rectangle,draw, fill=mymagenta!10, minimum size=1mm] (char){};}}}
\newcommand{\sourcesplit}{\raisebox{2.3pt}{\protect\tikz[baseline=(char.base)]{\protect\node[shape=rectangle,draw, left color=mymagenta!30, shading angle=45, right color=blue!10, anchor=north, rectangle, minimum size=1mm] (char){};}}}
\newcommand{\sourcepick}{\raisebox{2.3pt}{\protect\tikz[baseline=(char.base)]{\protect\node[shape=rectangle,draw, left color=mymagenta!30, shading angle=45, right color=orange!20, minimum size=1mm] (char){};}}}
\newcommand{\splitnode}{\raisebox{2.3pt}{\protect\tikz[baseline=(char.base)]{\protect\node[shape=rectangle,draw, fill=blue!10, minimum size=1mm] (char){};}}}
\newcommand{\regularCopy}{\raisebox{2.3pt}{\protect\tikz[baseline=(char.base)]{\protect\node[shape=rectangle,draw, fill=yellow!20, minimum size=1mm] (char){};}}}
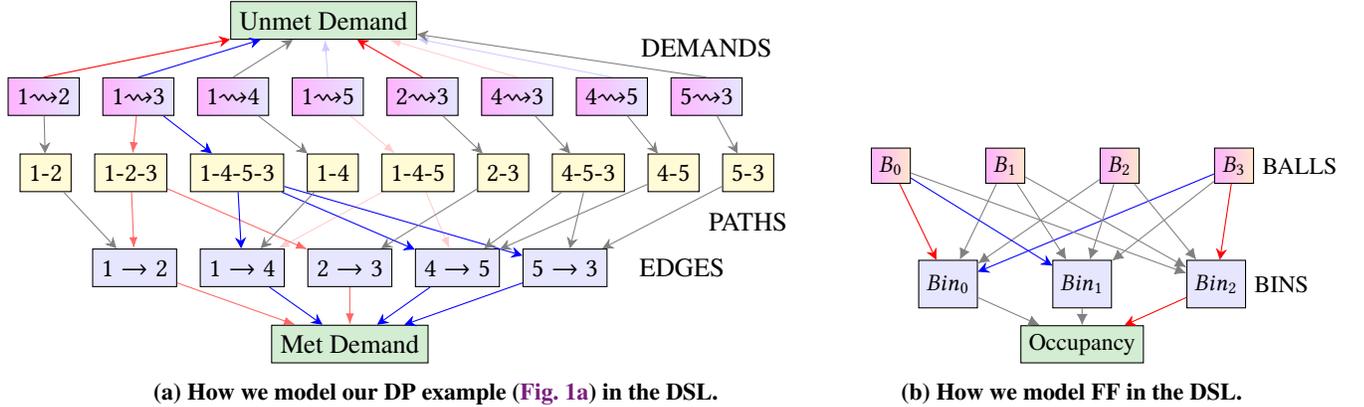
\begin{figure*}
    \begin{subfigure}[b]{0.64\textwidth}
    \begin{tikzpicture}[node distance=0.15cm and 0.3cm, 
        met_demand/.style={draw, fill=mymagenta!10, rectangle, minimum size=0.5cm, align=center},
        demand/.style={draw, shading = axis, left color=mymagenta!30, shading angle=45, right color=blue!10, anchor=north, rectangle, minimum size=0.5cm, align=center},
        paths/.style={draw, fill=yellow!20, rectangle, minimum size=0.5cm, align=center},
        edges/.style={draw, fill=blue!10, rectangle, minimum size=0.5cm, align=center},
        arrow/.style={-{Stealth[length=1.5mm, width=1.5mm]}, draw=gray}
        ]

        \tikzset{
            arrowstyle_gray/.style={-{Latex[length=1.5mm]}, draw=gray},
            arrowstyle_red/.style={-{Stealth[length=1.5mm, width=1.5mm]}, draw=red},
            arrowstyle_red_muted/.style={-{Latex[length=1.5mm]}, draw=red, opacity=0.6},
            arrowstyle_red_muted2/.style={-{Latex[length=1.5mm]}, draw=red, opacity=0.2},
            arrowstyle_blue/.style={-{Stealth[length=1.5mm, width=1.5mm]}, draw=blue},
            arrowstyle_blue_muted/.style={-{Latex[length=1.5mm]}, draw=blue, opacity=0.6},
            arrowstyle_blue_muted2/.style={-{Latex[length=1.5mm]}, draw=blue, opacity=0.2}
        }
        \node[met_demand, fill=forestgreen!20] (unmet) at (0,0) {Unmet Demand};
        \node[demand, below left=0.5cm and 2cm of unmet] (1x2) {$1$$\rightsquigarrow$$2$};
        \node[demand, right=of 1x2] (1x3) {$1$$\rightsquigarrow$$3$};
        \node[demand, right=of 1x3] (1x4) {$1$$\rightsquigarrow$$4$};
        \node[demand, right=of 1x4] (1x5) {$1$$\rightsquigarrow$$5$};
        \node[demand, right=of 1x5] (2x3) {$2$$\rightsquigarrow$$3$};
        \node[demand, right=of 2x3] (4x3) {$4$$\rightsquigarrow$$3$};
        \node[demand, right=of 4x3] (4x5) {$4$$\rightsquigarrow$$5$};
        \node[demand, right=of 4x5] (5x3) {$5$$\rightsquigarrow$$3$};
        \node[above=of 5x3] (demands) {DEMANDS};

        \node[paths, below right=0.5cm and -0.8cm of 1x2] (1-2) {$1$-$2$};
        \node[paths, right=of 1-2] (1-2-3) {$1$-$2$-$3$};
        \node[paths, right=of 1-2-3] (1-4-5-3) {$1$-$4$-$5$-$3$};
        \node[paths, right=of 1-4-5-3] (1-4) {$1$-$4$};
        \node[paths, right=of 1-4] (1-4-5) {$1$-$4$-$5$};
        \node[paths, right=of 1-4-5] (2-3) {$2$-$3$};
        \node[paths, right=of 2-3] (4-5-3) {$4$-$5$-$3$};
        \node[paths, right=of 4-5-3] (4-5) {$4$-$5$};
        \node[paths, right=of 4-5] (5-3) {$5$-$3$};

        \node[below=of 5-3] (paths) {PATHS};

        \node[edges, below right=.75cm and -1cm of 1-2-3] (e1x2) {$1\rightarrow2$};
        \node[edges, right=of e1x2] (e1x4) {$1\rightarrow4$};
        \node[edges, right=of e1x4] (e2x3) {$2\rightarrow3$};
        \node[edges, right=of e2x3] (e4x5) {$4\rightarrow5$};
        \node[edges, right=of e4x5] (e5x3) {$5\rightarrow3$};

        \node[right=of e5x3] (edges) {EDGES};

        \node[met_demand, fill=forestgreen!20, below=.5cm of e2x3] (met)  {Met Demand};

        \draw[arrowstyle_red] (1x2.north) -- (unmet);
        \draw[arrowstyle_blue] (1x3.north) -- (unmet);
        \draw[arrow] (1x4.north) -- (unmet);
        \draw[arrowstyle_blue_muted2] (1x5.north) -- (unmet);
        \draw[arrowstyle_red] (2x3.north) -- (unmet);
        \draw[arrowstyle_red_muted2] (4x3.north) -- (unmet);
        \draw[arrowstyle_blue_muted2] (4x5.north) -- (unmet);
        \draw[arrow] (5x3.north) -- (unmet);

        \draw[arrow] (1x2) -- (1-2);
        \draw[arrowstyle_red_muted] (1x3) -- (1-2-3);
        \draw[arrowstyle_blue] (1x3) -- (1-4-5-3);
        \draw[arrow] (1x4) -- (1-4);
        \draw[arrowstyle_red_muted2] (1x5) -- (1-4-5);
        \draw[arrow] (2x3) -- (2-3);
        \draw[arrow] (4x3) -- (4-5-3);
        \draw[arrow] (4x5) -- (4-5);
        \draw[arrow] (5x3) -- (5-3);

        \draw[arrow] (1-2) -- (e1x2);
        \draw[arrowstyle_red_muted] (1-2-3) -- (e1x2);
        \draw[arrowstyle_red_muted] (1-2-3) -- (e2x3);
        \draw[arrowstyle_blue] (1-4-5-3) -- (e1x4);
        \draw[arrowstyle_blue] (1-4-5-3) -- (e4x5);
        \draw[arrowstyle_blue] (1-4-5-3) -- (e5x3);
        \draw[arrow] (1-4) -- (e1x4);
        \draw[arrowstyle_red_muted2] (1-4-5) -- (e1x4);
        \draw[arrowstyle_red_muted2] (1-4-5) -- (e4x5);
        \draw[arrow] (2-3) -- (e2x3);
        \draw[arrow] (4-5-3) -- (e4x5);
        \draw[arrow] (4-5-3) -- (e5x3);
        \draw[arrow] (4-5) -- (e4x5);
        \draw[arrow] (5-3) -- (e5x3);

        \draw[arrowstyle_red_muted] (e1x2) -- (met);
        \draw[arrowstyle_blue] (e1x4) -- (met);
        \draw[arrowstyle_red_muted] (e2x3) -- (met);
        \draw[arrowstyle_blue] (e4x5) -- (met);
        \draw[arrowstyle_blue] (e5x3) -- (met);

    \end{tikzpicture}
    \caption{How we model our DP example (\autoref{f:pinning_issues}) in the DSL.}
    \label{fig:dsl_DP}
    \end{subfigure}
    \begin{subfigure}[b]{0.3\textwidth}
        \begin{tikzpicture}[node distance=0.25cm and 1.1cm, scale=0.9, every node/.style={transform shape}, auto]  
        \tikzset{
            Istyle_source/.style={draw, shading = axis, left color=mymagenta!30, shading angle=45, right color=orange!20, anchor=north, rectangle, minimum size=0.5cm, align=center},
        	Istyle/.style={rectangle, draw=black, fill=mymagenta!30, text centered, minimum height=2em},
        	arrowstyle/.style={-{Latex[length=1.5mm]}, draw=black}  
        }  
        
        \node[Istyle_source] (B0) {$B_0$};  
        \node[Istyle_source, right=of B0] (B1) {$B_1$};  
        \node[Istyle_source, right=of B1] (B2) {$B_2$};  
        \node[Istyle_source, right=of B2] (B3) {$B_3$}; 
        \node[node distance=0.2cm and 0.0cm, right=of B3] (BALLS) {BALLS};
        
        \tikzstyle{box} = [circle, draw, text centered, minimum height=2em]  
        \tikzstyle{occ} = [rectangle, draw, text centered, minimum height=1.5em]  
        
        \node[Istyle, fill=blue!10] (Bin0) at ($(B0)!0.5!(B1)-(0, 1.75cm)$) {$Bin_0$};   
        \node[Istyle, right=of Bin0, fill=blue!10] (Bin1) {$Bin_1$};  
        \node[Istyle, right=of Bin1, fill=blue!10] (Bin2) {$Bin_2$};
        \node[node distance=0.2cm and 0.0cm, right=of Bin2] (BINS) {BINS};
        \node[occ, fill=forestgreen!20, below=of Bin1] (Occ) {Occupancy};  
        
        \tikzset{
            arrowstyle_gray/.style={-{Latex[length=1.5mm, width=1.5mm]}, draw=gray},
            arrowstyle_red/.style={-{Stealth[length=1.5mm, width=1.5mm]}, draw=red},
            arrowstyle3/.style={-{Latex[length=1.5mm]}, draw=red, opacity=0.5},
            arrowstyle_blue/.style={-{Stealth[length=1.5mm, width=1.5mm]}, draw=blue},
            arrowstyle5/.style={-{Latex[length=1.5mm]}, draw=blue, opacity=0.5}
        }
        
        \foreach \from/\to/\style in {B0/Bin0/arrowstyle_red, B1/Bin0/arrowstyle_gray, B2/Bin0/arrowstyle_gray, B3/Bin0/arrowstyle_blue,  
            B0/Bin1/arrowstyle_blue, B1/Bin1/arrowstyle_gray, B2/Bin1/arrowstyle_gray,
            B3/Bin1/arrowstyle_gray,  
            B0/Bin2/arrowstyle_gray, B1/Bin2/arrowstyle_gray, B2/Bin2/arrowstyle_gray, B3/Bin2/arrowstyle_red}  
        {  
            \draw[\style] (\from) -- (\to);  
        }
        \foreach \from/\style in {Bin0, Bin1}  
        {  
        	\draw[arrowstyle_gray] (\from) -- (Occ);  
        }
        \draw[arrowstyle_red] (Bin2) -- (Occ);
    \end{tikzpicture}  
    \caption{How we model FF in the DSL.}
    \label{fig:dsl_vbp}
    \end{subfigure}
    \caption{Encoding heuristics in our DSL. We show sink nodes in \sink; source nodes enforcing behavior of split nodes in~\sourcesplit  \ and source nodes enforcing behavior of pick nodes in~\sourcepick; copy nodes in \regularCopy; and split nodes with limited outgoing capacity in \splitnode. The edge colors show type 2 explanations: more intense red (blue) edges show there are more samples in the subspace that only the heuristic (optimal) uses. In (a), DP uses the shortest path for the demand between $1$$\rightsquigarrow$$3$ and the optimal does not. In (b), we see FF places a large ball ($B_0$) in the first bin, causing it to have to place the last ball differently, too. We used 3000 samples for each explanation. \sysname took 20 minutes to produce each figure.}
    
    

\end{figure*}

We propose $\mathcal{X}$Plain (\autoref{fig::proposal}). 
Users describe the heuristic and benchmark through its \textbf{domain specific language} (\autoref{sec::dsl}). 
The main purpose of this domain-specific language (DSL) is to \emph{concretely define} the behavior of the heuristic and benchmark, which allows automated systems to analyze, compare, and explain their behavior.
The \textbf{compiler} translates the DSL into low-level optimization constructs.

The \textbf{adversarial subspace generator}(\autoref{sec::sub-generator}) generates a set of contiguous subspaces where the inputs in each subspace cause the heuristic to underperform and the \textbf{significance checker} filters the outputs and ensures the subspaces are statistically significant~---~it checks that the inputs that fall into these subspaces produce higher gaps compared to those that do not with statistical significance.

The \textbf{explainer}~(\autoref{sec::explainer}) describes how the heuristic's actions differ from the benchmark in each contiguous subspace for a given problem instance. The \textbf{generalizer}~(\autoref{sec:statistics}) extrapolates from these instance-based observations to produce the properties of the inputs and the instance that cause the heuristic to underperform. 
It uses instance-based explanations across many instances to do so~---~we use the \textbf{instance generator} to create such instances.

\subsection{The domain-specific language}
\label{sec::dsl}

To auto-generate the information we described in~\autoref{sec::case} we need a DSL to concretely encode the heuristic and benchmark algorithms. 
We need a DSL that: (1) can represent diverse heuristics; (2) we can use to automatically compile into optimizations that we can efficiently solve (those that existing solvers support and that do not introduce too many additional constraints and variables compared to hand-written models); and (3) is easy and intuitive to use. 


We design an abstraction based on \emph{network flow problems}~\cite{bertsimas}. Network flow problems are optimizations that, given a set of sources and destinations, optimize how to route traffic to respect capacity constraints, maximize link utilization, etc. 
Network flow problems impose two key constraints: the total flow on each link should be below the link capacity, and what comes into a node should go out (flow conservation).

There are advantages to using network flow problems: they have an intuitive graph representation~\cite{bertsimas}~---~operators know how to reason about the flow of traffic through such graphs; we can easily translate them into convex optimization or feasibility problems~\cite{bertsimas}; and they have many variants which we can use and build upon.

We can use the network flow model and extend it through a set of new ``node behaviors'' to ensure we can apply it to a broad class of heuristics. Node behaviors are a set of constraints that operate on the flows coming in and going out of each node: ``split nodes'' (enforce flow conservation constraints); ``pick nodes'' (enforce flow conservation constraints but only allow flow on a single outgoing edge); ``copy nodes'' (copy the flow that comes in onto all of their outgoing edges); ``source'' and ``sink'' nodes (produce or consume traffic); etc. 
A node can enforce multiple behaviors simultaneously. 
We include node behaviors that do not enforce flow conservation constraints (such as the ``copy nodes") or capacity constraints by default so that we can model a broad set of heuristics. 
Users can also add metadata to each node or edge, which we can use later to improve the explanations we produce.

Users encode the problem, the heuristic, and the benchmark in the DSL in abstract terms. For example, to model VBP they specify that the problem operates over (abstract) sequences of different node types that correspond to the balls and bins in the VBP problem. Users also encode the actions the heuristic and the optimal can make in terms of the relationship between the different sequences of nodes and the edges that connect them and rules that govern how flow can traverse from one node to the next. To analyze a specific instance of the VBP problem, users input the number of balls and bins and then \sysname concretizes the encoding (we show a concretized example with 4 one-dimensional balls and 3 bins in~\autoref{fig:dsl_vbp}).

Our DSL allows us to model the examples from prior work. 
We can model DP with split, source, and sink nodes (\autoref{fig:dsl_DP}), and we use ``pick nodes'' with limited capacity that only allow a ball to be assigned to a \emph{single} bin (\autoref{fig:dsl_vbp}) to model FF. 


\ifshowproof
We prove that we can represent \emph{any} linear or mixed integer problem through a small set of node behaviors (our abstraction is sufficient) in~\autoref{sec:proof}.
\fi

We can easily compile node behaviors into efficient optimizations. 
Our encoding allows us to solve the optimization faster compared to the hand-coded optimization:  our DSL allows us to find redundant constraints and variables.
This, in turn, reduces the number of variables and constraints MetaOpt adds in its re-writes\footnote[1]{Gurobi's pre-solve can also do this, but it changes the variable names, making it hard to connect them back to the original problem.}. 
We have implemented a complete DSL in a LINQ~\cite{linq}-style language: compared to the original MetaOpt implementation, the compiled DSL analyzes our DP example 4.3$\times$ faster. 
MetaOpt does not re-write FF, and we do not provide any run-time gains in that case.



\noindent \textbf{Open questions.} 
We can describe any heuristic that MetaOpt can analyze in our DSL. 
To support other analyzers (e.g.,~\cite{venkatarxive}) we may need to change our compiler and add other node behaviors.
We also need to understand what metadata the user can (or should) provide to enable $\mathcal{X}$Plain. 
This may require a co-design with $\mathcal{X}$Plain's other components. 

Although we have proved that any mixed integer program can be mapped to our DSL (\ifshowproof \autoref{sec:proof}\fi), that does not mean such a mapping is the most efficient representation of the heuristic in the DSL: we may achieve better performance if we model the heuristic directly in the DSL. We need further research to formalize and guide users in how to do so and optimize their representations.

\subsection{The adversarial subspace generator}
\label{sec::sub-generator}

Random search cannot find adversarial subspaces (it may not even find an adversarial point~\cite{metaopt}). We propose an algorithm where we extrapolate from the heuristic analyzer's output and: (1) use the analyzer to find an adversarial example; (2) find the adversarial subspace around that example; (3) exclude that subspace and repeat until we can no longer find an adversarial example (where the heuristic significantly underperforms) outside all of the subspaces we have found so far. 

To find each adversarial subspace, we first find a rough candidate region: we sample in a cubic area around the initial adversarial point given by a heuristic analyzer and expand our sampling area based on the density of adversarial (bad) samples we find in each direction. We define these ``directions'' based on where the sub-cube (slice) lies with respect to the initial adversarial point that MetaOpt found.
We stop when the density of bad samples drops in all possible expansion directions (\autoref{fig:cubes}).

We go ``slice by slice" when we investigate the cubic region around the initial bad sample because the adversarial subspace may not be uniformly spread around the initial point. 
We extend our sampling regions only around the slices where the density of bad samples is high. 
We pick the number of samples we use based on the DKW inequality~\cite{massart1990tight}. 

These subspace boundaries we have so far are not exact: how big we pick our slices and how much we expand them in each iteration influence how many false positives fall into the subspace. 
We refine the subspace based on an idea from prior work in diagnosis~\cite{chen2004failure}. 
We train a regression tree that predicts the performance gap on samples in our rough subspace. 
The predicates that form the path that starts at the root of this tree and reaches the leaf that contains the initial bad sample more accurately describe the subspace (\autoref{fig:decision_tree_vpb}).

\newcommand{\red}{{\protect\tikz[baseline=-0.5ex]\fill[red] (0,-0.1em) -- (0.25em,0.4em) -- (0.5em,-0.1em) -- cycle;}}
\newcommand{\blue}{{\protect\tikz[baseline=-0.5ex]\fill[blue] (0,0) circle (0.25em);}}

\begin{figure*}
   \begin{subfigure}[b]{0.29\textwidth}
        \begin{tikzpicture}
            \begin{axis}[
                title={},
                xlabel={$x$},
                ylabel={$y$},
                grid=both,
                width=\textwidth,
                height=\textwidth, 
                axis line style={draw=none}, 
                xticklabels={}, 
                yticklabels={}, 
                ylabel=\empty,
                xlabel=\empty
            ]
            \addplot [
                domain=0:1,
                samples=2,
                thick,
            ] coordinates {(-1,-1) (1,-1) (1,1) (-1,1) (-1,-1)};
            \addplot [
                domain=0:1,
                samples=2,
                fill=green,
                fill opacity=0.2 
            ] coordinates {(0,0) (1,0) (1,-1) (0,-1) (0,0)};
            \addplot [
                domain=0:1,
                samples=2,
                fill=green,
                fill opacity=0.2 
            ] coordinates {(0,0) (1.5,0) (1.5,1.5) (0,1.5) (0,0)};
            \addplot [
                domain=0:1,
                samples=2,
                thick
            ] coordinates {(0,-1.5) (1.5,-1.5) (1.5,1.5) (0,1.5) (0,-1.5)};
            \addplot [
                domain=0:1,
                samples=2,
                thick
            ] coordinates {(1.5,0) (2, 0) (2,2) (0,2) (0, 1.5)};
            \addplot [
                domain=0:1,
                samples=2,
                thick,
                dashed
            ] coordinates {(0,-1.5) (0,1.5)};
            \addplot [
                domain=0:1,
                samples=2,
                thick,
                dashed
            ] coordinates {(-1,0) (1.5,0)};
            \addplot+[
                only marks,
                mark=*,
                mark options={scale=0.5, fill=blue, draw=blue}
            ] coordinates {(0.870,  0.019)
                           (0.771,  0.761)
                           (0.461, -0.406)
                           (0.771, -0.959)
                           (0.295, -0.290)
                           (0.570, -0.858)};
            \addplot+[
                only marks,
                mark=*,
                mark options={scale=0.5, fill=blue, draw=blue}
            ] coordinates {(1.123, -0.010)
                           (1.172, -0.503)
                           (1.398, -0.802)
                           (1.480, -0.292)
                           (1.343, -0.053)
                           (0.156, -1.356)
                           (1.302, -1.283)
                           (0.335, -1.153)
                           (0.663, -1.338)
                           (1.395, -0.318)
                           (1.075, -0.501)
                           (1.223, -0.474)
                           (1.054, -0.868)
                           (0.925, -1.152)
                           (1.462, -1.089)
                           (0.457, -1.164)
                           (0.440, -1.323)
                           (1.050, -1.166)};
            
            \addplot+[
                only marks,
                mark=*,
                mark options={scale=0.5, fill=blue, draw=blue}
            ] coordinates {(-0.176, -0.555)
                           (-0.226,  0.800)
                           (-0.662, -0.530)
                           (-0.914, -0.385)
                           (-0.558, -0.424)
                           (-0.916,  0.799)
                           (-0.771,  0.448)
                           (-0.069, -0.432)
                           (-0.224,  0.012)
                           (-0.381, -0.255)
                           (-0.123,  0.997)
                           (-0.303,  0.575)
                           (-0.085, -0.206)
                           (-0.960, -0.298)
                           (-0.895,  0.636)
                           (-0.090, -0.288)
                           (-0.812,  0.165)
                           (-0.313,  0.318)
                           (-0.395, -0.244)
                           (-0.693,  0.800)
                           (-0.768,  0.423)
                           (-0.247, -0.884)
                           (-0.224, -0.808)
                           (-0.422,  0.919)
                           (-0.351,  0.898)
                           (-0.625,  0.427)
                           (-0.221, -0.070)
                           (-0.921, -0.717)
                           (-0.091,  0.805)
                           (-0.023, -0.631)
                           (-0.905, -0.478)
                           (-0.617, -0.854)
                           (-0.400,  0.071)
                           (-0.056,  0.994)
                           (-0.329, -0.975)
                           (-0.484,  0.420)
                           (-0.336,  0.794)
                           (-0.696, -0.180)
                           (-0.405,  0.269)
                           (-0.334,  0.917)
                           (-0.572, -0.563)
                           (-0.332, -0.370)
                           (-0.978, -0.525)
                           (-0.253, -0.762)
                           (-0.684, -0.575)
                           (-0.199, -0.387)
                           (-0.369,  0.109)
                           (-0.015, -0.078)
                           (-0.015,  0.036)
                           (-0.053, -0.067)
                           (-0.102, -0.265)
                           (-0.205,  0.355)
                           (-0.571, -0.066)
                           (-0.394,  0.042)
                           (-0.450,  0.022)
                           (-0.827,  0.323)
                           (-0.321, -0.633)
                           (-0.643, -0.923)
                           (-0.298,  0.534)
                           (-0.427,  0.728)
                           (-0.884,  0.081)
                           (-0.941,  0.184)
                           (-0.900, -0.871)
                           (-0.111, -0.215)
                           (-0.012,  0.509)
                           (-0.870, -0.383)
                           (-0.963, -0.916)
                           (-0.536, -0.289)
                           (0.1, 1.3)
                           (-0.944, -0.256)
                           (-0.431,  0.868)};
                \addplot+[
                    only marks,
                    mark=*,
                    mark options={scale=0.5, fill=blue, draw=blue}
                ] coordinates {(1.920, 1.743)
                               (1.804, 0.732)
                               (0.031, 1.507)
                               (1.785, 0.013)
                               (0.327, 1.609)
                               (1.261, 1.745)
                               (1.761, 1.840)
                               (0.676, 1.652)
                               (1.566, 1.247)
                               (0.693, 1.712)
                               (1.691, 0.750)
                               (1.664, 0.458)
                               (1.616, 0.981)
                               (1.605, 0.199)
                               (0.939, 1.918)
                               (0.557, 1.791)
                               (1.804, 0.844)
                               (1.995, 1.014)
                               (0.765, 1.729)
                               (0.402, 1.952)
                               (0.647, 1.552)
                               (0.257, 1.901)
                               (0.049, 1.622)
                               (0.946, 1.670)
                               (1.820, 0.393)
                               (0.813, 1.675)
                               (1.657, 0.663)
                               (1.881, 1.219)
                               (1.799, 0.087)
                               (0.828, 1.926)
                               (1.283, 1.520)
                               (1.544, 1.225)
                               (1.982, 0.955)
                               (1.082, 1.905)
                               (1.237, 1.837)
                               (1.561, 1.832)
                               (1.836, 0.908)
                               (1.534, 0.914)
                               (1.715, 0.994)
                               (1.817, 1.849)};
                
            \addplot+[
                only marks,
                mark=triangle*,
                mark options={scale=0.9, fill=red, draw=red}
            ] coordinates {(0, 0)};
            \addplot+[
                only marks,
                mark=triangle*,
                mark options={scale=0.9, fill=red, draw=red}
            ] coordinates {(-0.809, -0.739)
                           (-0.775, -0.458)
                           (-0.639,  0.507)
                           (-0.439,  0.407)
                           (-0.640, -0.687)
                           (-0.235, -0.372)
                           (1.156, -0.356)
                           (0.480, -1.429)
                           (1.75, 1.65)
                           };
            \addplot+[
                only marks,
                mark=triangle*,
                mark options={scale=0.9, fill=red, draw=red}
            ] coordinates {(1.006, 0.637)
                           (1.017, 0.739)
                           (1.060, 0.503)
                           (1.363, 0.890)
                           (1.368, 0.832)
                           (1.462, 0.040)
                           (1.447, 0.229)
                           (1.083, 0.125)
                           (1.325, 0.312)
                           (1.433, 1.411)
                           (0.513, 1.084)
                           (1.24, 1.303)
                           (0.45, 1.247)
                           (0.842, 1.086)
                           (0.603, 1.090)
                           (0.555, 1.347)
                           (0.163, 1.018)
                           (0.640, 1.490)};

            \addplot+[
                only marks,
                mark=triangle*,
                mark options={scale=0.9, fill=red, draw=red}
            ] coordinates {(0.194,  0.763)
                           (0.947, -0.114)
                           (0.744, -0.940)
                           (0.497,  0.399)
                           (0.243, -0.194)
                           (0.914,  0.896)
                           (0.033, -0.716)
                           (0.765, -0.930)
                           (0.688, -0.990)
                           (0.284, -0.631)
                           (0.125,  0.283)
                           (0.973, -0.911)
                           (0.707, -0.335)
                           (0.196,  0.742)
                           (0.695,  0.407)
                           (0.576, -0.997)
                           (0.083, -0.316)
                           (0.749,  0.556)
                           (0.269, -0.225)
                           (0.980, -0.600)
                           (0.859, -0.838)
                           (0.443,  0.826)
                           (0.367,  0.477)
                           (0.766, -0.524)
                           (0.932,  0.395)
                           (0.110, -0.186)
                           (0.066,  0.948)
                           (0.546, -0.966)
                           (0.254,  0.417)
                           (0.543,  0.386)
                           (0.498,  0.977)
                           (0.826,  0.973)
                           (0.653,  0.428)
                           (0.828,  0.015)
                           (0.058, -0.042)
                           (0.846,  0.943)
                           (0.354, -0.353)
                           (0.829,  0.790)
                           (0.184, -0.396)
                           (0.343,  0.115)
                           (0.304,  0.437)
                           (0.366, -0.355)
                           (0.297,  0.531)
                           (0.120,  0.053)
                           (0.400, -0.330)
                           (0.250,  0.723)
                           (0.891, -0.897)
                           (0.222, -0.530)
                           (0.846,  0.881)
                           (0.767, -0.116)
                           (0.601, -0.433)
                           (0.763,  0.227)
                           (0.916,  0.536)
                           (0.838, -0.972)
                           (0.959,  0.328)
                           (0.896, -0.009)
                           (0.809, -0.483)
                           (0.561, -0.943)
                           (0.287,  0.421)
                           (0.756, -0.884)
                           (0.324, -0.984)
                           (0.268, -0.583)
                           (0.756, -0.387)
                           (0.127, -0.916)
                           (0.183, -0.537)
                           (0.790,  0.295)
                           (0.885,  0.443)
                           (0.162, -0.079)
                           (0.538, -0.571)
                           (0.457, -0.072)};

            \end{axis}
        \end{tikzpicture}
        \caption{Identifying dense adversarial slices.}
        \label{fig:cubes}
    \end{subfigure}
    \begin{subfigure}[b]{0.27\textwidth}
    \centering
        \hspace*{-1em} 
        \begin{tikzpicture}[node distance=0.4cm and 1.4cm, scale=0.8, every node/.style={transform shape}, auto]  
        \tikzset{  
        	Istyle/.style={rectangle, draw=black, text centered, minimum height=2em},  
        	arrowstyle/.style={-{Latex[length=1.5mm]}, draw=black},
                arrowstyle1/.style={-{Latex[length=1.5mm]}, draw=gray!70}
        }  
        
        \node[Istyle, fill=orange!25] (Root) {$\sum_{n=0}^{3}B_n \leq 1.5$};  
        
        \node[Istyle, fill=orange!5] (V0) at ($(Root)-(1.2cm, 1.2cm)$) {$Gap = 1\%$};   
        \node[Istyle, fill=orange!50] (V1) at ($(Root)-(-1.2cm, 1.2cm)$) {$B_1  <= 0.5$};
        \node[Istyle, fill=orange!100] (V2) at ($(V1)-(1.2cm, 1.2cm)$) {$Gap = 25\%$};
        \node[Istyle, fill=orange!5] (V3) at ($(V1)-(-1.2cm, 1.2cm)$) {$Gap = 3\%$};

        \foreach \from/\to/\style in {Root/V0/arrowstyle1, Root/V1/arrowstyle, V1/V2/arrowstyle, V1/V3/arrowstyle1}  
        {  
            \draw[\style] (\from) -- (\to);  
        }

    \end{tikzpicture}  
    \caption{\small Refinement by regression tree.}
    \label{fig:decision_tree_vpb}
    \end{subfigure}
\begin{subfigure}{0.33\textwidth}
    \footnotesize
    \begin{flushleft}
    Adversarial subspaces: $\cup_i D_i$
     \begin{flalign*}
    & D_i = \cup_j \left\{ \textbf{X} \in \mathbb{R^+}^4 \ \Big| \ \left[
        \begin{array}{c}
        \textbf{A} \\
        \textbf{T}_i
        \end{array}\right] \textbf{X} \leq \left[
        \begin{array}{c}
        \textbf{C}_{j}^{i} \\
        \textbf{V}_i
        \end{array}\right] \right\} &
        \end{flalign*}
    \begin{flalign*}
    & \textbf{A} = \left[
    \begin{array}{c}
    \textbf{I}_{4 \times 4} \\
    - \textbf{I}_{4 \times 4} \\
    \end{array}\right], \quad
    \textbf{X} = \left[ B_0 \quad B_1 \quad B_2 \quad B_3 \right]^T & %
    \end{flalign*}
   
    
    \begin{flalign*}
    & D_0: \textbf{C}_{0}^{0} = \left[
    0.01 \; 0.51 \; 0.51 \; 0.51 \; 0 \; -0.49 \; -0.49 \; -0.49
    \right]^T & 
    \end{flalign*}
\\
    \[
    \textbf{T}_0 = \left[
    \begin{array}{cccc}
    -1 & -1 & -1 & -1 \\
    0 & 1 & 0 & 0
    \end{array}\right],
    \quad
    \textbf{V}_0 = \left[-1.5 \quad 0.5\right]^T
    \]
    \caption{The adversarial subspaces for FF. }
    \label{fig:mathematical_vbp}
    \end{flushleft}
\end{subfigure}

    \caption{ The adversarial subspace generator: (a) finds a rough subspace and separates bad samples (\red) from good ones (\blue); (b) it trains a regression tree on these samples and uses it to refine the subspace and produces (c). We show the first subspace ($D_0$) for our FF example in (c). Here, $C^i_j$ encodes the rough subspace and $T_i$ and $V_i$ the path in the regression tree.}
\end{figure*}
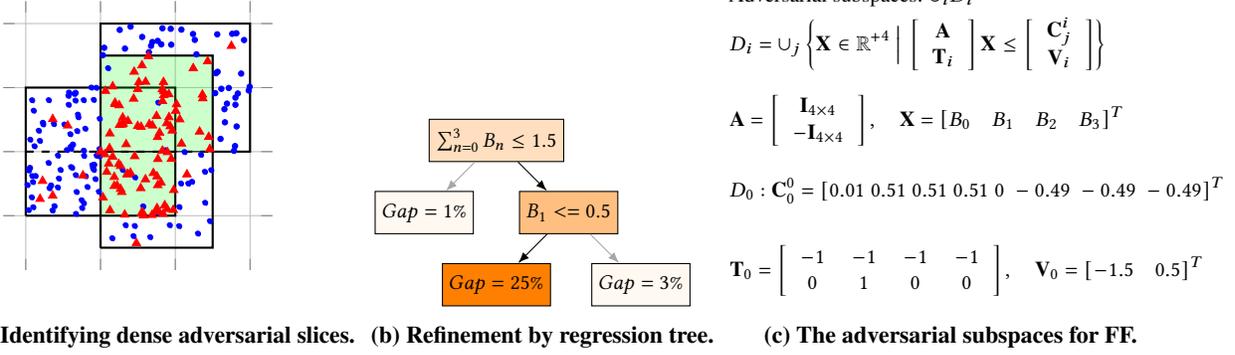


\noindent \textbf{The significance checker \label{sec:statistics}} ensures the subspaces we find are statistically significant: the points in a subspace cause a higher performance gap compared to those immediately outside it. We only report those subspaces with a low-p-value (less than $0.05$) as adversarial.

We use the  Wilcoxon signed-rank test ~\cite{wilcoxon1945individual}, which allows for dependant samples~---~the subspace fully describes what points are inside and what points are not (the samples in the two pools are dependent). We find subspaces for DP and VBP with p-values $2 \times 10^{-60}$ and $8 \times 10^{-11}$, respectively.

Our approach allows us to find all \emph{statistically significant} subspaces that meet our exploration granularity. If we do not include an adversarial input in a subspace (if it is outside of the region we explored), the analyzer will find it in the next iteration. Users can control \sysname's ability to find all adversarial scenarios: they can use smaller cube-sizes to explore the space in more detail but it comes at the cost of a slower runtime. They can also elect to include those parts of the initial subspaces \sysname finds (before we apply the decision tree) as part of MetaOpt's decision space (if they do so they need to include the number of times they are willing to re-examine an area to avoid an infinite cycle~---~there may be regions that are not statistically significant and \sysname would revisit them if they contain a input instance that produces a high gap).

\noindent\textbf{Open questions.} The decision tree helps us identify predicates (of the form $f \ge t$ where $f$ is a feature and $t$ a threshold) that describe a subspace. What features we train the tree on influence what predicates we can get. On small instances we can use raw inputs but on larger instances this would require a deep decision tree to fully describe the space~---~the output becomes computationally more difficult to use in the next step (step (3) above). We need to define functions $\mathcal{F}(\mathcal{I})$ of the input $\mathcal{I}$ that allow us to describe these subspaces efficiently and which we can use in the analyzers to execute step (3) (i.e., where we exclude a subspace and re-run the analyzer). 

It may be better if we apply the adversarial subspace generator (steps (1)-(3) above) directly to the ``projected'' input space: where each function $\mathcal{F}(\mathcal{I})$ describes one dimension of the $m$-dimensional space (note, $m$ need not be the same as the dimensions of the input space $\mathcal{I}$). If the space defined by the adversarial subspaces is sparse this approach may allow us to find these adversarial subspaces more efficiently.

We may need additional mechanisms to help scale \sysname~---~it may take a long time to find adversarial subspaces if we analyze a large problem instance or if there are many disjoint subspaces. 



\subsection{The explainer}
\label{sec::explainer}
We hypothesize that the inputs in a contiguous subspace share the same root cause for why they cause the heuristic to underperform. 
This is where a network-flow-based DSL explicitly encoding the decisions of the heuristic and the benchmark algorithm proves useful. 
We run samples from within each contiguous subspace through the DSL and score edges based on if: (1) both the benchmark and the heuristic send flow on that edge (score = 0); (2) only the benchmark sends flow (score = 1); or (3) only the heuristic sends flow (score = -1).

Such a ``heatmap'' of the differences between the benchmark and the heuristic shows how inputs in the subspace interfere with the heuristic. In~\autoref{fig:dsl_DP}, in a given subspace with $3000$ samples, all pinnable demands share the \emph{same} shortest path (red arrows in 1-2-3 path), and the optimal routes them through alternative paths (blue arrows in 1-4-5-3 path).

\noindent \textbf{Open questions.} As the instance size (the scale of the problem we want to analyze) grows, the above heatmap may become harder to interpret. 
We need mechanisms that allow us to summarize the information in this heatmap in a way that the user can interpret and use to improve their heuristic. 

The heuristic and benchmark also differ in how much flow they route on each edge. We need to define the appropriate data structure to represent this information to a user so that they are interpretable and actionable.

\subsection{The generalizer and instance generator}
\label{sec:generalizer}

We can enable operators to improve their heuristics or know when to apply mitigations if we can extrapolate from the type 1 and 2 explanations to form type 3: what properties in the adversarial inputs cause the heuristic to underperform and what aspects of the problem instance exacerbate it? We need to find trends across instance-based information and find an instance-agnostic explanation for why the heuristic underperformed.



To discover patterns, we need to consider a diverse set of instances and identify trends in the outputs of the subspace generator and the explainer. We build an \emph{instance generator} that uses the problem description in the DSL to create such instances and feeds them into the pipeline.

We imagine the generalizer would contain a ``grammar'' that uses the metadata the user provides through the DSL along with the network flow structure to describe trends in the instance-based explanations. 
For example, one may consider this predicate from a hypothetical grammar:

{\small
\begin{align*}
\texttt{increasing($\mathcal{P}$)}: \forall a, b \mid a,b \in \mathcal{P} ~ \& ~ |a| \ge |b| \rightarrow gap(a) \ge gap(b)
\end{align*}
}

With such a grammar, a generalizer can go through the observations on the samples the instance generator produced and check if the predicates in the grammar are statistically significant. For example, if $\mathcal{P}$ describes the set of shortest paths of pinnable demands in DP, the generalizer might produce \texttt{increasing($\mathcal{P}$)} for why DP underperforms~---~this predicate suggests that the gap is larger when the shortest path of the pinnable demands is longer.



\noindent \textbf{Open questions.} 
One may envision a solution similar to enumerative synthesis~\cite{gulwani2011synthesis, Alur2013SyntaxGuidedS, huang2020reconciling}, which searches through the grammar, finds all predicates that hold for a particular heuristic, and forms clauses that explain the heuristic's behavior. We need more work to define the generalizer's grammar and how to build valid clauses from them.






\section{Related work}
\label{sec:related}

To our knowledge, this is the first work that focuses on a \emph{general} framework to provide more insights into the outputs of heuristic analysis tools~\cite{metaopt,venkatarxive} and provides an explainability feature for these tools. We build on prior work:

\noindent \textbf{Domain customized performance analyzers.} The work we do in $\mathcal{X}$Plain also applies to custom performance analyzers, which only apply to specific heuristics~\cite{CCAC, mina, venkatstarvation}.

\noindent \textbf{Explainable AI.} $\mathcal{X}$Plain resembles prior work in explainable AI, which provided more context around what different ML models predict~\cite{ribeiro2016should, lundberg2017unified, wachter2017counterfactual}. Parts of our solution (including the three types) are inspired by these works~\cite{hci, automl, phillips2021four}.

\noindent \textbf{Enumerative Synthesis.} This field generates programs that meet a specification through systematic enumeration of possible program candidates~\cite{gulwani2011synthesis, Alur2013SyntaxGuidedS, huang2020reconciling}. We believe these ideas can help us to design the generalizer.

\noindent \textbf{Large Language Models (LLMs).} we may be able to use LLMs~\cite{10256783} for various parts of our designs these include: to generate the DSL, to summarize Type 2 explanations, and to generate the grammer we need to produce Type 3 explanations. But LLMs are prone to hallucination~\cite{huang2023survey, liu2024exploring} and also require additional step-by-step mechanisms to guide them~\cite{wang2023plan, kambhampati2024llms}. We may be able to build a natural language interface that can help us automatically generate the DSL. Such an interface will enable non-experts to more easily use \sysname. This, too, is an interesting topic for future work. 

\section{Acknowledgements}

We would like to thank Basmira Nushi, Ishai Menache, Konstantina Mellou, Luke Marshall, Amin Khodaverdian, Chenning Li, Joe Chandler, and Weiyang Wang for their valuable comments. We also than the HotNets program committee for their valuable feedback.

\newpage
\bibliographystyle{ACM-Reference-Format} 

\begin{thebibliography}{49}


\ifx \showCODEN    \undefined \def \showCODEN     #1{\unskip}     \fi
\ifx \showDOI      \undefined \def \showDOI       #1{#1}\fi
\ifx \showISBNx    \undefined \def \showISBNx     #1{\unskip}     \fi
\ifx \showISBNxiii \undefined \def \showISBNxiii  #1{\unskip}     \fi
\ifx \showISSN     \undefined \def \showISSN      #1{\unskip}     \fi
\ifx \showLCCN     \undefined \def \showLCCN      #1{\unskip}     \fi
\ifx \shownote     \undefined \def \shownote      #1{#1}          \fi
\ifx \showarticletitle \undefined \def \showarticletitle #1{#1}   \fi
\ifx \showURL      \undefined \def \showURL       {\relax}        \fi
\providecommand\bibfield[2]{#2}
\providecommand\bibinfo[2]{#2}
\providecommand\natexlab[1]{#1}
\providecommand\showeprint[2][]{arXiv:#2}

\bibitem[Agarwal et~al\mbox{.}(2022)]%
        {venkat1}
\bibfield{author}{\bibinfo{person}{Anup Agarwal}, \bibinfo{person}{Venkat Arun}, \bibinfo{person}{Devdeep Ray}, \bibinfo{person}{Ruben Martins}, {and} \bibinfo{person}{Srinivasan Seshan}.} \bibinfo{year}{2022}\natexlab{}.
\newblock \showarticletitle{Automating network heuristic design and analysis}. In \bibinfo{booktitle}{\emph{Proceedings of the 21st ACM Workshop on Hot Topics in Networks}}. \bibinfo{pages}{8--16}.
\newblock


\bibitem[Agarwal et~al\mbox{.}(2024)]%
        {venkatNSDI24}
\bibfield{author}{\bibinfo{person}{Anup Agarwal}, \bibinfo{person}{Venkat Arun}, \bibinfo{person}{Devdeep Ray}, \bibinfo{person}{Ruben Martins}, {and} \bibinfo{person}{Srinivasan Seshan}.} \bibinfo{year}{2024}\natexlab{}.
\newblock \showarticletitle{Towards provably performant congestion control}. In \bibinfo{booktitle}{\emph{21st USENIX Symposium on Networked Systems Design and Implementation (NSDI 24)}}. \bibinfo{publisher}{USENIX Association}, \bibinfo{address}{Santa Clara, CA}, \bibinfo{pages}{951--978}.
\newblock
\showISBNx{978-1-939133-39-7}
\urldef\tempurl%
\url{https://www.usenix.org/conference/nsdi24/presentation/agarwal-anup}
\showURL{%
\tempurl}


\bibitem[Alur et~al\mbox{.}(2013)]%
        {Alur2013SyntaxGuidedS}
\bibfield{author}{\bibinfo{person}{Rajeev Alur}, \bibinfo{person}{Rastislav Bodik}, \bibinfo{person}{Garvit Juniwal}, \bibinfo{person}{Milo~M.K. Martin}, \bibinfo{person}{Mukund Raghothaman}, \bibinfo{person}{Sanjit~A. Seshia}, \bibinfo{person}{Rishabh Singh}, \bibinfo{person}{Armando Solar-Lezama}, \bibinfo{person}{Emina Torlak}, {and} \bibinfo{person}{Abhishek Udupa}.} \bibinfo{year}{2013}\natexlab{}.
\newblock \showarticletitle{Syntax-Guided Synthesis}.
\newblock \bibinfo{journal}{\emph{Proceedings of the International Conference on Formal Methods in Computer-Aided Design}} (\bibinfo{year}{2013}).
\newblock
\showISBNx{9780983567837}
\urldef\tempurl%
\url{https://doi.org/10.1109/FMCAD.2013.6679385}
\showDOI{\tempurl}


\bibitem[Amershi et~al\mbox{.}(2019)]%
        {hci}
\bibfield{author}{\bibinfo{person}{Saleema Amershi}, \bibinfo{person}{Andrew Begel}, \bibinfo{person}{Christian Bird}, \bibinfo{person}{Robert DeLine}, \bibinfo{person}{Harald Gall}, \bibinfo{person}{Ece Kamar}, \bibinfo{person}{Nachiappan Nagappan}, \bibinfo{person}{Besmira Nushi}, {and} \bibinfo{person}{Thomas Zimmermann}.} \bibinfo{year}{2019}\natexlab{}.
\newblock \showarticletitle{Software engineering for machine learning: A case study}. In \bibinfo{booktitle}{\emph{2019 IEEE/ACM 41st International Conference on Software Engineering: Software Engineering in Practice (ICSE-SEIP)}}. IEEE, \bibinfo{pages}{291--300}.
\newblock


\bibitem[Arashloo et~al\mbox{.}(2023)]%
        {mina}
\bibfield{author}{\bibinfo{person}{Mina~Tahmasbi Arashloo}, \bibinfo{person}{Ryan Beckett}, {and} \bibinfo{person}{Rachit Agarwal}.} \bibinfo{year}{2023}\natexlab{}.
\newblock \showarticletitle{Formal Methods for Network Performance Analysis}. In \bibinfo{booktitle}{\emph{20th USENIX Symposium on Networked Systems Design and Implementation (NSDI 23)}}. \bibinfo{pages}{645--661}.
\newblock


\bibitem[Arun et~al\mbox{.}(2022)]%
        {venkatstarvation}
\bibfield{author}{\bibinfo{person}{Venkat Arun}, \bibinfo{person}{Mohammad Alizadeh}, {and} \bibinfo{person}{Hari Balakrishnan}.} \bibinfo{year}{2022}\natexlab{}.
\newblock \showarticletitle{Starvation in end-to-end congestion control}. In \bibinfo{booktitle}{\emph{Proceedings of the ACM SIGCOMM 2022 Conference}} (Amsterdam, Netherlands) \emph{(\bibinfo{series}{SIGCOMM '22})}. \bibinfo{publisher}{Association for Computing Machinery}, \bibinfo{address}{New York, NY, USA}, \bibinfo{pages}{177–192}.
\newblock
\showISBNx{9781450394208}
\urldef\tempurl%
\url{https://doi.org/10.1145/3544216.3544223}
\showDOI{\tempurl}


\bibitem[Arun et~al\mbox{.}(2021)]%
        {CCAC}
\bibfield{author}{\bibinfo{person}{Venkat Arun}, \bibinfo{person}{Mina~Tahmasbi Arashloo}, \bibinfo{person}{Ahmed Saeed}, \bibinfo{person}{Mohammad Alizadeh}, {and} \bibinfo{person}{Hari Balakrishnan}.} \bibinfo{year}{2021}\natexlab{}.
\newblock \showarticletitle{Toward formally verifying congestion control behavior}. In \bibinfo{booktitle}{\emph{Proceedings of the 2021 ACM SIGCOMM 2021 Conference}}. \bibinfo{pages}{1--16}.
\newblock


\bibitem[Arzani et~al\mbox{.}(2021)]%
        {automl}
\bibfield{author}{\bibinfo{person}{Behnaz Arzani}, \bibinfo{person}{Kevin Hsieh}, {and} \bibinfo{person}{Haoxian Chen}.} \bibinfo{year}{2021}\natexlab{}.
\newblock \showarticletitle{Interpretable feedback for AutoML and a proposal for domain-customized AutoML for networking}. In \bibinfo{booktitle}{\emph{Proceedings of the 20th ACM Workshop on Hot Topics in Networks}}. \bibinfo{pages}{53--60}.
\newblock


\bibitem[Barbalho et~al\mbox{.}(2023)]%
        {vmlifetime}
\bibfield{author}{\bibinfo{person}{Hugo Barbalho}, \bibinfo{person}{Patricia Kovaleski}, \bibinfo{person}{Beibin Li}, \bibinfo{person}{Luke Marshall}, \bibinfo{person}{Marco Molinaro}, \bibinfo{person}{Abhisek Pan}, \bibinfo{person}{Eli Cortez}, \bibinfo{person}{Matheus Leao}, \bibinfo{person}{Harsh Patwari}, \bibinfo{person}{Zuzu Tang}, {et~al\mbox{.}}} \bibinfo{year}{2023}\natexlab{}.
\newblock \showarticletitle{Virtual machine allocation with lifetime predictions}.
\newblock \bibinfo{journal}{\emph{Proceedings of Machine Learning and Systems}}  \bibinfo{volume}{5} (\bibinfo{year}{2023}).
\newblock


\bibitem[Beckett et~al\mbox{.}(2017)]%
        {minesweeper}
\bibfield{author}{\bibinfo{person}{Ryan Beckett}, \bibinfo{person}{Aarti Gupta}, \bibinfo{person}{Ratul Mahajan}, {and} \bibinfo{person}{David Walker}.} \bibinfo{year}{2017}\natexlab{}.
\newblock \showarticletitle{A General Approach to Network Configuration Verification}. In \bibinfo{booktitle}{\emph{Proceedings of the Conference of the ACM Special Interest Group on Data Communication}} (Los Angeles, CA, USA) \emph{(\bibinfo{series}{SIGCOMM '17})}. \bibinfo{publisher}{ACM}, \bibinfo{address}{New York, NY, USA}, \bibinfo{pages}{155--168}.
\newblock
\showISBNx{978-1-4503-4653-5}
\urldef\tempurl%
\url{https://doi.org/10.1145/3098822.3098834}
\showDOI{\tempurl}


\bibitem[Bertsimas and Tsitsiklis(1997)]%
        {bertsimas}
\bibfield{author}{\bibinfo{person}{Dimitris Bertsimas} {and} \bibinfo{person}{John~N Tsitsiklis}.} \bibinfo{year}{1997}\natexlab{}.
\newblock \bibinfo{booktitle}{\emph{Introduction to linear optimization}}. Vol.~\bibinfo{volume}{6}.
\newblock \bibinfo{publisher}{Athena Scientific Belmont, MA}.
\newblock


\bibitem[Boyd and Vandenberghe(2004)]%
        {boyd2004convex}
\bibfield{author}{\bibinfo{person}{Stephen~P Boyd} {and} \bibinfo{person}{Lieven Vandenberghe}.} \bibinfo{year}{2004}\natexlab{}.
\newblock \bibinfo{booktitle}{\emph{Convex optimization}}.
\newblock \bibinfo{publisher}{Cambridge university press}.
\newblock


\bibitem[Chen et~al\mbox{.}(2004)]%
        {chen2004failure}
\bibfield{author}{\bibinfo{person}{Mike Chen}, \bibinfo{person}{Alice~X Zheng}, \bibinfo{person}{Jim Lloyd}, \bibinfo{person}{Michael~I Jordan}, {and} \bibinfo{person}{Eric Brewer}.} \bibinfo{year}{2004}\natexlab{}.
\newblock \showarticletitle{Failure diagnosis using decision trees}. In \bibinfo{booktitle}{\emph{International Conference on Autonomic Computing, 2004. Proceedings.}} IEEE, \bibinfo{pages}{36--43}.
\newblock


\bibitem[Fayaz et~al\mbox{.}(2016)]%
        {era}
\bibfield{author}{\bibinfo{person}{Seyed~K. Fayaz}, \bibinfo{person}{Tushar Sharma}, \bibinfo{person}{Ari Fogel}, \bibinfo{person}{Ratul Mahajan}, \bibinfo{person}{Todd Millstein}, \bibinfo{person}{Vyas Sekar}, {and} \bibinfo{person}{George Varghese}.} \bibinfo{year}{2016}\natexlab{}.
\newblock \showarticletitle{Efficient Network Reachability Analysis Using a Succinct Control Plane Representation}. In \bibinfo{booktitle}{\emph{12th {USENIX} Symposium on Operating Systems Design and Implementation ({OSDI} 16)}}. \bibinfo{publisher}{{USENIX} Association}, \bibinfo{address}{Savannah, GA}, \bibinfo{pages}{217--232}.
\newblock
\showISBNx{978-1-931971-33-1}
\urldef\tempurl%
\url{https://www.usenix.org/conference/osdi16/technical-sessions/presentation/fayaz}
\showURL{%
\tempurl}


\bibitem[Gember-Jacobson et~al\mbox{.}(2016)]%
        {arc}
\bibfield{author}{\bibinfo{person}{Aaron Gember-Jacobson}, \bibinfo{person}{Raajay Viswanathan}, \bibinfo{person}{Aditya Akella}, {and} \bibinfo{person}{Ratul Mahajan}.} \bibinfo{year}{2016}\natexlab{}.
\newblock \showarticletitle{Fast Control Plane Analysis Using an Abstract Representation}. In \bibinfo{booktitle}{\emph{Proceedings of the 2016 ACM SIGCOMM Conference}} (Florianopolis, Brazil) \emph{(\bibinfo{series}{SIGCOMM '16})}. \bibinfo{publisher}{ACM}, \bibinfo{address}{New York, NY, USA}, \bibinfo{pages}{300--313}.
\newblock
\showISBNx{978-1-4503-4193-6}
\urldef\tempurl%
\url{https://doi.org/10.1145/2934872.2934876}
\showDOI{\tempurl}


\bibitem[Goel et~al\mbox{.}(2023)]%
        {venkatarxive}
\bibfield{author}{\bibinfo{person}{Saksham Goel}, \bibinfo{person}{Benjamin Mikek}, \bibinfo{person}{Jehad Aly}, \bibinfo{person}{Venkat Arun}, \bibinfo{person}{Ahmed Saeed}, {and} \bibinfo{person}{Aditya Akella}.} \bibinfo{year}{2023}\natexlab{}.
\newblock \showarticletitle{Quantitative verification of scheduling heuristics}.
\newblock \bibinfo{journal}{\emph{arXiv preprint arXiv:2301.04205}} (\bibinfo{year}{2023}).
\newblock


\bibitem[Grumberg et~al\mbox{.}(2004)]%
        {allsat2}
\bibfield{author}{\bibinfo{person}{Orna Grumberg}, \bibinfo{person}{Assaf Schuster}, {and} \bibinfo{person}{Avi Yadgar}.} \bibinfo{year}{2004}\natexlab{}.
\newblock \showarticletitle{Memory Efficient All-Solutions SAT Solver and Its Application for Reachability Analysis}. In \bibinfo{booktitle}{\emph{Formal Methods in Computer-Aided Design}}, \bibfield{editor}{\bibinfo{person}{Alan~J. Hu} {and} \bibinfo{person}{Andrew~K. Martin}} (Eds.). \bibinfo{publisher}{Springer Berlin Heidelberg}, \bibinfo{address}{Berlin, Heidelberg}, \bibinfo{pages}{275--289}.
\newblock
\showISBNx{978-3-540-30494-4}


\bibitem[Gulwani et~al\mbox{.}(2011)]%
        {gulwani2011synthesis}
\bibfield{author}{\bibinfo{person}{Sumit Gulwani}, \bibinfo{person}{Susmit Jha}, \bibinfo{person}{Ashish Tiwari}, {and} \bibinfo{person}{Ramarathnam Venkatesan}.} \bibinfo{year}{2011}\natexlab{}.
\newblock \showarticletitle{Synthesis of Loop-free Programs}. In \bibinfo{booktitle}{\emph{Proceedings of the 32nd ACM SIGPLAN Conference on Programming Language Design and Implementation (PLDI)}}. ACM, \bibinfo{pages}{62--73}.
\newblock


\bibitem[Horn et~al\mbox{.}(2017)]%
        {deltanet}
\bibfield{author}{\bibinfo{person}{Alex Horn}, \bibinfo{person}{Ali Kheradmand}, {and} \bibinfo{person}{Mukul Prasad}.} \bibinfo{year}{2017}\natexlab{}.
\newblock \showarticletitle{Delta-net: Real-time Network Verification Using Atoms}. In \bibinfo{booktitle}{\emph{14th {USENIX} Symposium on Networked Systems Design and Implementation ({NSDI} 17)}}. \bibinfo{publisher}{{USENIX} Association}, \bibinfo{address}{Boston, MA}, \bibinfo{pages}{735--749}.
\newblock
\showISBNx{978-1-931971-37-9}
\urldef\tempurl%
\url{https://www.usenix.org/conference/nsdi17/technical-sessions/presentation/horn-alex}
\showURL{%
\tempurl}


\bibitem[Huang et~al\mbox{.}(2020)]%
        {huang2020reconciling}
\bibfield{author}{\bibinfo{person}{Kangjing Huang}, \bibinfo{person}{Xiaokang Qiu}, \bibinfo{person}{Peiyuan Shen}, {and} \bibinfo{person}{Yanjun Wang}.} \bibinfo{year}{2020}\natexlab{}.
\newblock \showarticletitle{Reconciling enumerative and deductive program synthesis}. In \bibinfo{booktitle}{\emph{Proceedings of the 41st ACM SIGPLAN Conference on Programming Language Design and Implementation}}. \bibinfo{pages}{1159--1174}.
\newblock


\bibitem[Huang et~al\mbox{.}(2023)]%
        {huang2023survey}
\bibfield{author}{\bibinfo{person}{Lei Huang}, \bibinfo{person}{Weijiang Yu}, \bibinfo{person}{Weitao Ma}, \bibinfo{person}{Weihong Zhong}, \bibinfo{person}{Zhangyin Feng}, \bibinfo{person}{Haotian Wang}, \bibinfo{person}{Qianglong Chen}, \bibinfo{person}{Weihua Peng}, \bibinfo{person}{Xiaocheng Feng}, \bibinfo{person}{Bing Qin}, {et~al\mbox{.}}} \bibinfo{year}{2023}\natexlab{}.
\newblock \showarticletitle{A survey on hallucination in large language models: Principles, taxonomy, challenges, and open questions}.
\newblock \bibinfo{journal}{\emph{arXiv preprint arXiv:2311.05232}} (\bibinfo{year}{2023}).
\newblock


\bibitem[Jayaraman et~al\mbox{.}(2019)]%
        {fibverifier}
\bibfield{author}{\bibinfo{person}{Karthick Jayaraman}, \bibinfo{person}{Nikolaj Bjorner}, \bibinfo{person}{Jitu Padhye}, \bibinfo{person}{Amar Agrawal}, \bibinfo{person}{Ashish Bhargava}, \bibinfo{person}{Paul-Andre~C Bissonnette}, \bibinfo{person}{Shane Foster}, \bibinfo{person}{Andrew Helwer}, \bibinfo{person}{Mark Kasten}, \bibinfo{person}{Ivan Lee}, \bibinfo{person}{Anup Namdhari}, \bibinfo{person}{Haseeb Niaz}, \bibinfo{person}{Aniruddha Parkhi}, \bibinfo{person}{Hanukumar Pinnamraju}, \bibinfo{person}{Adrian Power}, \bibinfo{person}{Neha~Milind Raje}, {and} \bibinfo{person}{Parag Sharma}.} \bibinfo{year}{2019}\natexlab{}.
\newblock \showarticletitle{Validating Datacenters at Scale}. In \bibinfo{booktitle}{\emph{Proceedings of the ACM Special Interest Group on Data Communication}} (Beijing, China) \emph{(\bibinfo{series}{SIGCOMM '19})}. \bibinfo{publisher}{ACM}, \bibinfo{address}{New York, NY, USA}, \bibinfo{pages}{200--213}.
\newblock
\showISBNx{978-1-4503-5956-6}
\urldef\tempurl%
\url{https://doi.org/10.1145/3341302.3342094}
\showDOI{\tempurl}


\bibitem[Jayaraman et~al\mbox{.}(2014)]%
        {secguru}
\bibfield{author}{\bibinfo{person}{Karthick Jayaraman}, \bibinfo{person}{Nikolaj Bjørner}, \bibinfo{person}{Geoff Outhred}, {and} \bibinfo{person}{Charlie Kaufman}.} \bibinfo{year}{2014}\natexlab{}.
\newblock \bibinfo{booktitle}{\emph{Automated Analysis and Debugging of Network Connectivity Policies}}.
\newblock \bibinfo{type}{{T}echnical {R}eport} MSR-TR-2014-102. \bibinfo{institution}{Microsoft}.
\newblock


\bibitem[Kakarla et~al\mbox{.}(2020a)]%
        {groot}
\bibfield{author}{\bibinfo{person}{Siva Kesava~Reddy Kakarla}, \bibinfo{person}{Ryan Beckett}, \bibinfo{person}{Behnaz Arzani}, \bibinfo{person}{Todd Millstein}, {and} \bibinfo{person}{George Varghese}.} \bibinfo{year}{2020}\natexlab{a}.
\newblock \showarticletitle{GRoot: Proactive Verification of DNS Configurations}. In \bibinfo{booktitle}{\emph{Proceedings of the Annual Conference of the ACM Special Interest Group on Data Communication on the Applications, Technologies, Architectures, and Protocols for Computer Communication}} (Virtual Event, USA) \emph{(\bibinfo{series}{SIGCOMM ’20})}. \bibinfo{publisher}{Association for Computing Machinery}, \bibinfo{address}{New York, NY, USA}, \bibinfo{pages}{310–328}.
\newblock
\showISBNx{9781450379557}
\urldef\tempurl%
\url{https://doi.org/10.1145/3387514.3405871}
\showDOI{\tempurl}


\bibitem[Kakarla et~al\mbox{.}(2020b)]%
        {selfstarter}
\bibfield{author}{\bibinfo{person}{Siva Kesava~Reddy Kakarla}, \bibinfo{person}{Alan Tang}, \bibinfo{person}{Ryan Beckett}, \bibinfo{person}{Karthick Jayaraman}, \bibinfo{person}{Todd Millstein}, \bibinfo{person}{Yuval Tamir}, {and} \bibinfo{person}{George Varghese}.} \bibinfo{year}{2020}\natexlab{b}.
\newblock \showarticletitle{Finding Network Misconfigurations by Automatic Template Inference}. In \bibinfo{booktitle}{\emph{17th {USENIX} Symposium on Networked Systems Design and Implementation ({NSDI} 20)}}. \bibinfo{publisher}{{USENIX} Association}, \bibinfo{address}{Santa Clara, CA}, \bibinfo{pages}{999--1013}.
\newblock
\showISBNx{978-1-939133-13-7}
\urldef\tempurl%
\url{https://www.usenix.org/conference/nsdi20/presentation/kakarla}
\showURL{%
\tempurl}


\bibitem[Kambhampati et~al\mbox{.}(2024)]%
        {kambhampati2024llms}
\bibfield{author}{\bibinfo{person}{Subbarao Kambhampati}, \bibinfo{person}{Karthik Valmeekam}, \bibinfo{person}{Lin Guan}, \bibinfo{person}{Kaya Stechly}, \bibinfo{person}{Mudit Verma}, \bibinfo{person}{Siddhant Bhambri}, \bibinfo{person}{Lucas Saldyt}, {and} \bibinfo{person}{Anil Murthy}.} \bibinfo{year}{2024}\natexlab{}.
\newblock \showarticletitle{LLMs Can't Plan, But Can Help Planning in LLM-Modulo Frameworks}.
\newblock \bibinfo{journal}{\emph{arXiv preprint arXiv:2402.01817}} (\bibinfo{year}{2024}).
\newblock


\bibitem[Kazemian et~al\mbox{.}(2012)]%
        {hsa}
\bibfield{author}{\bibinfo{person}{Peyman Kazemian}, \bibinfo{person}{George Varghese}, {and} \bibinfo{person}{Nick McKeown}.} \bibinfo{year}{2012}\natexlab{}.
\newblock \showarticletitle{Header Space Analysis: Static Checking for Networks}. In \bibinfo{booktitle}{\emph{9th {USENIX} Symposium on Networked Systems Design and Implementation ({NSDI} 12)}}. \bibinfo{publisher}{{USENIX} Association}, \bibinfo{address}{San Jose, CA}, \bibinfo{pages}{113--126}.
\newblock
\showISBNx{978-931971-92-8}
\urldef\tempurl%
\url{https://www.usenix.org/conference/nsdi12/technical-sessions/presentation/kazemian}
\showURL{%
\tempurl}


\bibitem[Khurshid et~al\mbox{.}(2013)]%
        {veriflow}
\bibfield{author}{\bibinfo{person}{Ahmed Khurshid}, \bibinfo{person}{Xuan Zou}, \bibinfo{person}{Wenxuan Zhou}, \bibinfo{person}{Matthew Caesar}, {and} \bibinfo{person}{P.~Brighten Godfrey}.} \bibinfo{year}{2013}\natexlab{}.
\newblock \showarticletitle{VeriFlow: Verifying Network-Wide Invariants in Real Time}. In \bibinfo{booktitle}{\emph{Presented as part of the 10th {USENIX} Symposium on Networked Systems Design and Implementation ({NSDI} 13)}}. \bibinfo{publisher}{{USENIX}}, \bibinfo{address}{Lombard, IL}, \bibinfo{pages}{15--27}.
\newblock
\showISBNx{978-1-931971-00-3}
\urldef\tempurl%
\url{https://www.usenix.org/conference/nsdi13/technical-sessions/presentation/khurshid}
\showURL{%
\tempurl}


\bibitem[Liu et~al\mbox{.}(2024)]%
        {liu2024exploring}
\bibfield{author}{\bibinfo{person}{Fang Liu}, \bibinfo{person}{Yang Liu}, \bibinfo{person}{Lin Shi}, \bibinfo{person}{Houkun Huang}, \bibinfo{person}{Ruifeng Wang}, \bibinfo{person}{Zhen Yang}, {and} \bibinfo{person}{Li Zhang}.} \bibinfo{year}{2024}\natexlab{}.
\newblock \showarticletitle{Exploring and evaluating hallucinations in llm-powered code generation}.
\newblock \bibinfo{journal}{\emph{arXiv preprint arXiv:2404.00971}} (\bibinfo{year}{2024}).
\newblock


\bibitem[Lopes et~al\mbox{.}(2015)]%
        {nod}
\bibfield{author}{\bibinfo{person}{Nuno~P. Lopes}, \bibinfo{person}{Nikolaj Bj\o{}rner}, \bibinfo{person}{Patrice Godefroid}, \bibinfo{person}{Karthick Jayaraman}, {and} \bibinfo{person}{George Varghese}.} \bibinfo{year}{2015}\natexlab{}.
\newblock \showarticletitle{Checking Beliefs in Dynamic Networks}. In \bibinfo{booktitle}{\emph{Proceedings of the 12th USENIX Conference on Networked Systems Design and Implementation}} (Oakland, CA) \emph{(\bibinfo{series}{NSDI'15})}. \bibinfo{publisher}{USENIX Association}, \bibinfo{address}{USA}, \bibinfo{pages}{499–512}.
\newblock
\showISBNx{9781931971218}


\bibitem[Lundberg and Lee(2017)]%
        {lundberg2017unified}
\bibfield{author}{\bibinfo{person}{Scott~M Lundberg} {and} \bibinfo{person}{Su-In Lee}.} \bibinfo{year}{2017}\natexlab{}.
\newblock \showarticletitle{A unified approach to interpreting model predictions}.
\newblock \bibinfo{journal}{\emph{Advances in neural information processing systems}}  \bibinfo{volume}{30} (\bibinfo{year}{2017}).
\newblock


\bibitem[Mai et~al\mbox{.}(2011)]%
        {anteater}
\bibfield{author}{\bibinfo{person}{Haohui Mai}, \bibinfo{person}{Ahmed Khurshid}, \bibinfo{person}{Rachit Agarwal}, \bibinfo{person}{Matthew Caesar}, \bibinfo{person}{P.~Brighten Godfrey}, {and} \bibinfo{person}{Samuel~Talmadge King}.} \bibinfo{year}{2011}\natexlab{}.
\newblock \showarticletitle{Debugging the Data Plane with Anteater}.
\newblock \bibinfo{journal}{\emph{SIGCOMM Comput. Commun. Rev.}} \bibinfo{volume}{41}, \bibinfo{number}{4} (\bibinfo{date}{aug} \bibinfo{year}{2011}), \bibinfo{pages}{290–301}.
\newblock
\showISSN{0146-4833}
\urldef\tempurl%
\url{https://doi.org/10.1145/2043164.2018470}
\showDOI{\tempurl}


\bibitem[Massart(1990)]%
        {massart1990tight}
\bibfield{author}{\bibinfo{person}{P. Massart}.} \bibinfo{year}{1990}\natexlab{}.
\newblock \showarticletitle{The Tight Constant in the Dvoretzky-Kiefer-Wolfowitz Inequality}.
\newblock \bibinfo{journal}{\emph{The Annals of Probability}} \bibinfo{volume}{18}, \bibinfo{number}{3} (\bibinfo{year}{1990}), \bibinfo{pages}{1269--1283}.
\newblock
\urldef\tempurl%
\url{https://doi.org/10.1214/aop/1176990746}
\showDOI{\tempurl}


\bibitem[McMillan(2002)]%
        {allsat1}
\bibfield{author}{\bibinfo{person}{Ken~L. McMillan}.} \bibinfo{year}{2002}\natexlab{}.
\newblock \showarticletitle{Applying SAT Methods in Unbounded Symbolic Model Checking}. In \bibinfo{booktitle}{\emph{Computer Aided Verification}}, \bibfield{editor}{\bibinfo{person}{Ed~Brinksma} {and} \bibinfo{person}{Kim~Guldstrand Larsen}} (Eds.). \bibinfo{publisher}{Springer Berlin Heidelberg}, \bibinfo{address}{Berlin, Heidelberg}, \bibinfo{pages}{250--264}.
\newblock
\showISBNx{978-3-540-45657-5}


\bibitem[Namyar et~al\mbox{.}(2024)]%
        {metaopt}
\bibfield{author}{\bibinfo{person}{Pooria Namyar}, \bibinfo{person}{Behnaz Arzani}, \bibinfo{person}{Ryan Beckett}, \bibinfo{person}{Santiago Segarra}, \bibinfo{person}{Himanshu Raj}, \bibinfo{person}{Umesh Krishnaswamy}, \bibinfo{person}{Ramesh Govindan}, {and} \bibinfo{person}{Srikanth Kandula}.} \bibinfo{year}{2024}\natexlab{}.
\newblock \showarticletitle{Finding Adversarial Inputs for Heuristics using Multi-level Optimization}. In \bibinfo{booktitle}{\emph{21st USENIX Symposium on Networked Systems Design and Implementation (NSDI 24)}}. \bibinfo{publisher}{USENIX Association}, \bibinfo{address}{Santa Clara, CA}, \bibinfo{pages}{927--949}.
\newblock
\showISBNx{978-1-939133-39-7}
\urldef\tempurl%
\url{https://www.usenix.org/conference/nsdi24/presentation/namyar-finding}
\showURL{%
\tempurl}


\bibitem[Panigrahy et~al\mbox{.}(2011)]%
        {theoryVBP}
\bibfield{author}{\bibinfo{person}{Rina Panigrahy}, \bibinfo{person}{Kunal Talwar}, \bibinfo{person}{Lincoln Uyeda}, {and} \bibinfo{person}{Udi Wieder}.} \bibinfo{year}{2011}\natexlab{}.
\newblock \showarticletitle{Heuristics for vector bin packing}.
\newblock \bibinfo{journal}{\emph{research. microsoft. com}} (\bibinfo{year}{2011}).
\newblock


\bibitem[Phillips et~al\mbox{.}(2021)]%
        {phillips2021four}
\bibfield{author}{\bibinfo{person}{P~Jonathon Phillips}, \bibinfo{person}{P~Jonathon Phillips}, \bibinfo{person}{Carina~A Hahn}, \bibinfo{person}{Peter~C Fontana}, \bibinfo{person}{Amy~N Yates}, \bibinfo{person}{Kristen Greene}, \bibinfo{person}{David~A Broniatowski}, {and} \bibinfo{person}{Mark~A Przybocki}.} \bibinfo{year}{2021}\natexlab{}.
\newblock \showarticletitle{Four principles of explainable artificial intelligence}.
\newblock  (\bibinfo{year}{2021}).
\newblock


\bibitem[Ribeiro et~al\mbox{.}(2016)]%
        {ribeiro2016should}
\bibfield{author}{\bibinfo{person}{Marco~Tulio Ribeiro}, \bibinfo{person}{Sameer Singh}, {and} \bibinfo{person}{Carlos Guestrin}.} \bibinfo{year}{2016}\natexlab{}.
\newblock \showarticletitle{" Why should i trust you?" Explaining the predictions of any classifier}. In \bibinfo{booktitle}{\emph{Proceedings of the 22nd ACM SIGKDD international conference on knowledge discovery and data mining}}. \bibinfo{pages}{1135--1144}.
\newblock


\bibitem[Tang et~al\mbox{.}(2021)]%
        {campion}
\bibfield{author}{\bibinfo{person}{Alan Tang}, \bibinfo{person}{Siva Kesava~Reddy Kakarla}, \bibinfo{person}{Ryan Beckett}, \bibinfo{person}{Ennan Zhai}, \bibinfo{person}{Matt Brown}, \bibinfo{person}{Todd Millstein}, \bibinfo{person}{Yuval Tamir}, {and} \bibinfo{person}{George Varghese}.} \bibinfo{year}{2021}\natexlab{}.
\newblock \showarticletitle{Campion: Debugging Router Configuration Differences}. In \bibinfo{booktitle}{\emph{Proceedings of the 2021 ACM SIGCOMM 2021 Conference}} (Virtual Event, USA) \emph{(\bibinfo{series}{SIGCOMM '21})}. \bibinfo{publisher}{Association for Computing Machinery}, \bibinfo{address}{New York, NY, USA}, \bibinfo{pages}{748–761}.
\newblock
\showISBNx{9781450383837}
\urldef\tempurl%
\url{https://doi.org/10.1145/3452296.3472925}
\showDOI{\tempurl}


\bibitem[Tian et~al\mbox{.}(2019)]%
        {tian19safely}
\bibfield{author}{\bibinfo{person}{Bingchuan Tian}, \bibinfo{person}{Xinyi Zhang}, \bibinfo{person}{Ennan Zhai}, \bibinfo{person}{Hongqiang~Harry Liu}, \bibinfo{person}{Qiaobo Ye}, \bibinfo{person}{Chunsheng Wang}, \bibinfo{person}{Xin Wu}, \bibinfo{person}{Zhiming Ji}, \bibinfo{person}{Yihong Sang}, \bibinfo{person}{Ming Zhang}, \bibinfo{person}{Da Yu}, \bibinfo{person}{Chen Tian}, \bibinfo{person}{Haitao Zheng}, {and} \bibinfo{person}{Ben~Y. Zhao}.} \bibinfo{year}{2019}\natexlab{}.
\newblock \showarticletitle{Safely and Automatically Updating In-Network ACL Configurations with Intent Language}. In \bibinfo{booktitle}{\emph{Proceedings of the ACM Special Interest Group on Data Communication}} (Beijing, China) \emph{(\bibinfo{series}{SIGCOMM '19})}. \bibinfo{publisher}{Association for Computing Machinery}, \bibinfo{address}{New York, NY, USA}, \bibinfo{pages}{214–226}.
\newblock
\showISBNx{9781450359566}
\urldef\tempurl%
\url{https://doi.org/10.1145/3341302.3342088}
\showDOI{\tempurl}


\bibitem[Torgersen(2007)]%
        {linq}
\bibfield{author}{\bibinfo{person}{Mads Torgersen}.} \bibinfo{year}{2007}\natexlab{}.
\newblock \showarticletitle{Querying in C\# how language integrated query (LINQ) works}. In \bibinfo{booktitle}{\emph{Companion to the 22nd ACM SIGPLAN conference on Object-oriented programming systems and applications companion}}. \bibinfo{pages}{852--853}.
\newblock


\bibitem[Wachter et~al\mbox{.}(2017)]%
        {wachter2017counterfactual}
\bibfield{author}{\bibinfo{person}{Sandra Wachter}, \bibinfo{person}{Brent Mittelstadt}, {and} \bibinfo{person}{Chris Russell}.} \bibinfo{year}{2017}\natexlab{}.
\newblock \showarticletitle{Counterfactual explanations without opening the black box: Automated decisions and the GDPR}.
\newblock \bibinfo{journal}{\emph{Harv. JL \& Tech.}}  \bibinfo{volume}{31} (\bibinfo{year}{2017}), \bibinfo{pages}{841}.
\newblock


\bibitem[Wang et~al\mbox{.}(2023)]%
        {wang2023plan}
\bibfield{author}{\bibinfo{person}{Lei Wang}, \bibinfo{person}{Wanyu Xu}, \bibinfo{person}{Yihuai Lan}, \bibinfo{person}{Zhiqiang Hu}, \bibinfo{person}{Yunshi Lan}, \bibinfo{person}{Roy Ka-Wei Lee}, {and} \bibinfo{person}{Ee-Peng Lim}.} \bibinfo{year}{2023}\natexlab{}.
\newblock \showarticletitle{Plan-and-solve prompting: Improving zero-shot chain-of-thought reasoning by large language models}.
\newblock \bibinfo{journal}{\emph{arXiv preprint arXiv:2305.04091}} (\bibinfo{year}{2023}).
\newblock


\bibitem[Wilcoxon(1945)]%
        {wilcoxon1945individual}
\bibfield{author}{\bibinfo{person}{Frank Wilcoxon}.} \bibinfo{year}{1945}\natexlab{}.
\newblock \showarticletitle{Individual comparisons by ranking methods}.
\newblock \bibinfo{journal}{\emph{Biometrics bulletin}} \bibinfo{volume}{1}, \bibinfo{number}{6} (\bibinfo{year}{1945}), \bibinfo{pages}{80--83}.
\newblock


\bibitem[Woeginger(1997)]%
        {vbpAPX}
\bibfield{author}{\bibinfo{person}{Gerhard~J Woeginger}.} \bibinfo{year}{1997}\natexlab{}.
\newblock \showarticletitle{There is no asymptotic PTAS for two-dimensional vector packing}.
\newblock \bibinfo{journal}{\emph{Inform. Process. Lett.}} \bibinfo{volume}{64}, \bibinfo{number}{6} (\bibinfo{year}{1997}), \bibinfo{pages}{293--297}.
\newblock


\bibitem[Yan et~al\mbox{.}(2023)]%
        {10256783}
\bibfield{author}{\bibinfo{person}{Zheyu Yan}, \bibinfo{person}{Yifan Qin}, \bibinfo{person}{Xiaobo~Sharon Hu}, {and} \bibinfo{person}{Yiyu Shi}.} \bibinfo{year}{2023}\natexlab{}.
\newblock \showarticletitle{On the Viability of Using LLMs for SW/HW Co-Design: An Example in Designing CiM DNN Accelerators}. In \bibinfo{booktitle}{\emph{2023 IEEE 36th International System-on-Chip Conference (SOCC)}}. \bibinfo{pages}{1--6}.
\newblock
\urldef\tempurl%
\url{https://doi.org/10.1109/SOCC58585.2023.10256783}
\showDOI{\tempurl}


\bibitem[Yang and Lam(2016)]%
        {atomic}
\bibfield{author}{\bibinfo{person}{Hongkun Yang} {and} \bibinfo{person}{Simon~S. Lam}.} \bibinfo{year}{2016}\natexlab{}.
\newblock \showarticletitle{Real-time Verification of Network Properties Using Atomic Predicates}.
\newblock \bibinfo{journal}{\emph{IEEE/ACM Trans. Netw.}} \bibinfo{volume}{24}, \bibinfo{number}{2} (\bibinfo{date}{April} \bibinfo{year}{2016}), \bibinfo{pages}{887--900}.
\newblock
\showISSN{1063-6692}
\urldef\tempurl%
\url{https://doi.org/10.1109/TNET.2015.2398197}
\showDOI{\tempurl}


\bibitem[Yu et~al\mbox{.}(2014)]%
        {allsat3}
\bibfield{author}{\bibinfo{person}{Yinlei Yu}, \bibinfo{person}{Pramod Subramanyan}, \bibinfo{person}{Nestan Tsiskaridze}, {and} \bibinfo{person}{Sharad Malik}.} \bibinfo{year}{2014}\natexlab{}.
\newblock \showarticletitle{All-SAT Using Minimal Blocking Clauses}. In \bibinfo{booktitle}{\emph{2014 27th International Conference on VLSI Design and 2014 13th International Conference on Embedded Systems}}. \bibinfo{pages}{86--91}.
\newblock
\urldef\tempurl%
\url{https://doi.org/10.1109/VLSID.2014.22}
\showDOI{\tempurl}


\bibitem[Zhang et~al\mbox{.}(2020)]%
        {apkeep}
\bibfield{author}{\bibinfo{person}{Peng Zhang}, \bibinfo{person}{Xu Liu}, \bibinfo{person}{Hongkun Yang}, \bibinfo{person}{Ning Kang}, \bibinfo{person}{Zhengchang Gu}, {and} \bibinfo{person}{Hao Li}.} \bibinfo{year}{2020}\natexlab{}.
\newblock \showarticletitle{APKeep: Realtime Verification for Real Networks}. In \bibinfo{booktitle}{\emph{17th {USENIX} Symposium on Networked Systems Design and Implementation ({NSDI} 20)}}. \bibinfo{publisher}{{USENIX} Association}, \bibinfo{address}{Santa Clara, CA}, \bibinfo{pages}{241--255}.
\newblock
\showISBNx{978-1-939133-13-7}
\urldef\tempurl%
\url{https://www.usenix.org/conference/nsdi20/presentation/zhang-peng}
\showURL{%
\tempurl}


\end{thebibliography}

\ifshowproof
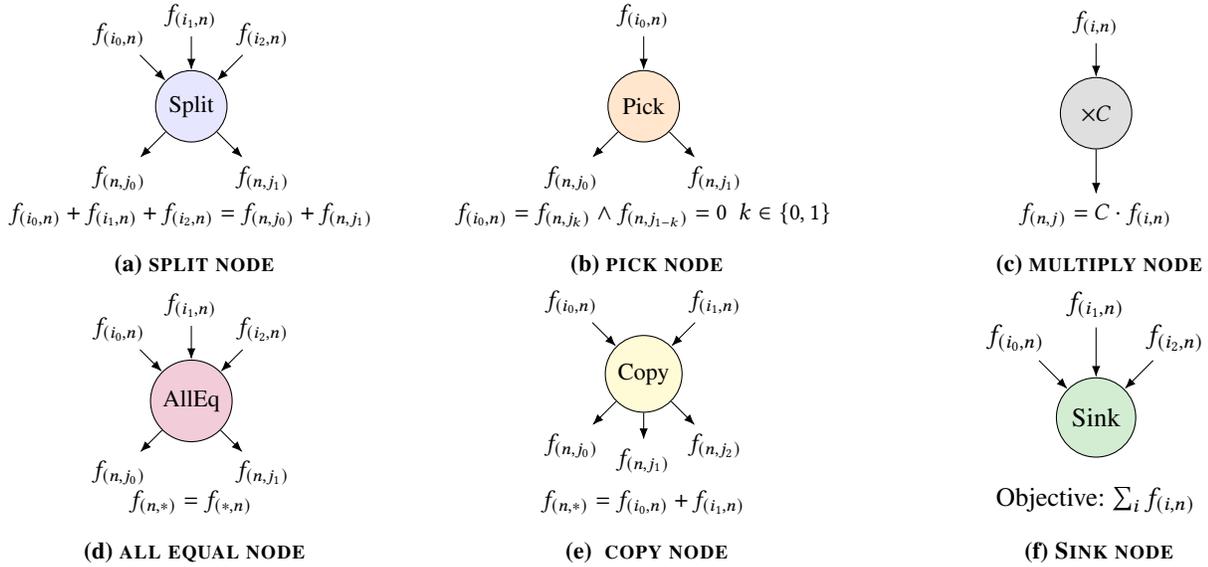
\begin{figure*}[t!]
    \centering
    \subcaptionbox{\label{fig:split}\textsc{split node}}[0.3\textwidth]{%
        \begin{tikzpicture}[baseline, node distance=1.5cm and 2cm, auto, scale=0.9, every node/.style={transform shape}]
            \tikzset{
                splitstyle/.style={circle, draw=black, fill=blue!10, text centered, minimum height=3em, minimum width=3em},
                arrowstyle/.style={-{Latex[length=1.5mm]}, draw=black},
            }
            \node[splitstyle] (Split) {Split};
            \node[above left of=Split] (In1) {$f_{(i_0,n)}$};
            \node[above right of=Split] (In2) {$f_{(i_2,n)}$};
            \node[above=0.5cm of Split] (In3) {$f_{(i_1,n)}$};
            \node[below left of=Split] (Out1) {$f_{(n,j_0)}$};
            \node[below right of=Split] (Out2) {$f_{(n,j_1)}$};
            \draw[arrowstyle] (In1) -- (Split);
            \draw[arrowstyle] (In2) -- (Split);
            \draw[arrowstyle] (In3) -- (Split);
            \draw[arrowstyle] (Split) -- (Out1);
            \draw[arrowstyle] (Split) -- (Out2);
            \node at ($(Split) - (0,1.3)$) [below] {$f_{(i_0,n)} + f_{(i_1,n)} + f_{(i_2,n)} = f_{(n, j_0)} + f_{(n, j_1)}$};
        \end{tikzpicture}
        \label{fig:split}
    }
    \hspace{0.5cm}
    \subcaptionbox{\label{fig:pick}\textsc{pick node}}[0.3\textwidth]{%
        \begin{tikzpicture}[baseline, node distance=1.5cm and 2cm, auto, scale=0.9, every node/.style={transform shape}]
            \tikzset{
                pickstyle/.style={circle, draw=black, fill=orange!20, text centered, minimum height=3em, minimum width=3em},
                arrowstyle/.style={-{Latex[length=1.5mm]}, draw=black},
            }
            \node[pickstyle] (Pick) {Pick};
            \node[above=0.5cm of Pick] (In) {$f_{(i_0,n)}$};
            \node[below left of=Pick] (Out1) {$f_{(n,j_0)}$};
            \node[below right of=Pick] (Out2) {$f_{(n,j_1)}$};
            \draw[arrowstyle] (In) -- (Pick);
            \draw[arrowstyle] (Pick) -- (Out1);
            \draw[arrowstyle] (Pick) -- (Out2);
            \node at ($(Pick) - (0,1.3)$) [below] {$f_{(i_0,n)} = f_{(n,j_{k})} \land f_{(n,j_{1-k})} = 0 ~~ k \in \{0, 1\} $};
        \end{tikzpicture}
        \label{fig:pick}
    }
    \hspace{0.5cm}
\subcaptionbox{\label{fig:mult}\textsc{multiply node}}[0.3\textwidth]{%
        \begin{tikzpicture}[baseline, node distance=1.5cm and 2cm, auto, scale=0.9, every node/.style={transform shape}]
            \tikzset{
                multstyle/.style={circle, draw=black, fill=gray!25, text centered, minimum height=3em, minimum width=3em},
                arrowstyle/.style={-{Latex[length=1.5mm]}, draw=black},
            }
            \node[multstyle] (Mult) {$\times C$};
            \node[above=0.5cm of Mult] (In) {$f_{(i,n)}$};
            \node[below of=Mult] (Out) {$f_{(n,j)} = C \cdot f_{(i,n)}$};
            \draw[arrowstyle] (In) -- (Mult);
            \draw[arrowstyle] (Mult) -- (Out);
        \end{tikzpicture}
        \label{fig:mult}
    }
    \hspace{0.5cm}
    \subcaptionbox{\label{fig:all_equal}\textsc{all equal node}}[0.3\textwidth]{%
        \begin{tikzpicture}[baseline, node distance=1.5cm and 2cm, auto, scale=0.9, every node/.style={transform shape}]
            \tikzset{
                allEqstyle/.style={circle, draw=black, fill=purple!20, text centered, minimum height=3em, minimum width=3em},
                arrowstyle/.style={-{Latex[length=1.5mm]}, draw=black},
            }
            \node[allEqstyle] (AllEq) {AllEq};
            \node[above left of=AllEq] (In1) {$f_{(i_0,n)}$};
            \node[above right of=AllEq] (In2) {$f_{(i_2,n)}$};
            \node[above=0.5cm of AllEq] (In3) {$f_{(i_1,n)}$};
            \node[below left of=AllEq] (Out1) {$f_{(n,j_0)}$};
            \node[below right of=AllEq] (Out2) {$f_{(n,j_1)}$};
            \draw[arrowstyle] (In1) -- (AllEq);
            \draw[arrowstyle] (In2) -- (AllEq);
            \draw[arrowstyle] (In3) -- (AllEq);
            \draw[arrowstyle] (AllEq) -- (Out1);
            \draw[arrowstyle] (AllEq) -- (Out2);
            \node at ($(AllEq) - (0,1.2)$) [below] {$f_{(n,*)} = f_{(*,n)}$};
        \end{tikzpicture}
        \label{fig:all_equal}
    }
    \hspace{0.5cm}
    \subcaptionbox{\label{fig:copy} \textsc{copy node}}
    [0.3\textwidth]{%
        \begin{tikzpicture}[baseline, node distance=1.5cm and 2cm, auto, scale=0.9, every node/.style={transform shape}]
            \tikzset{
                copystyle/.style={circle, draw=black, fill=yellow!20, text centered, minimum height=3em, minimum width=3em},
                arrowstyle/.style={-{Latex[length=1.5mm]}, draw=black},
            }
            \node[copystyle] (Copy) {Copy};
            \node[above left of=Copy] (In1) {$f_{(i_0,n)}$};
            \node[above right of=Copy] (In2) {$f_{(i_1,n)}$};
            \node[below left of=Copy] (Out1) {$f_{(n,j_0)}$};
            \node[below right of=Copy] (Out2) {$f_{(n,j_2)}$};
            \node[below=0.4cm of Copy] (Out3) {$f_{(n,j_1)}$};
            \draw[arrowstyle] (In1) -- (Copy);
            \draw[arrowstyle] (In2) -- (Copy);
            \draw[arrowstyle] (Copy) -- (Out1);
            \draw[arrowstyle] (Copy) -- (Out2);
            \draw[arrowstyle] (Copy) -- (Out3);
            \node at ($(Copy) - (0,1.6)$) [below] {$f_{(n,*)} = f_{(i_0,n)} + f_{(i_1,n)}$};
        \end{tikzpicture}
        \label{fig:copy}
    } \hspace{0.5cm}
    \subcaptionbox{\label{fig:sink}\textsc{Sink node}}[0.3\textwidth]{%
    \begin{tikzpicture}[baseline, node distance=1.5cm and 2cm, auto, scale=1, every node/.style={transform shape}]
        \tikzset{
            sinkstyle/.style={circle, draw=black, fill=forestgreen!20, text centered, minimum height=3em, minimum width=3em},
            arrowstyle/.style={-{Latex[length=1.5mm]}, draw=black},
        }
    
        \node[sinkstyle] (Sink) {Sink};
    
        \node[above left of=Sink] (In1) {$f_{(i_0,n)}$};
        \node[above of=Sink] (In2) {$f_{(i_1,n)}$};
        \node[above right of=Sink] (In3) {$f_{(i_2,n)}$};
    
        \draw[arrowstyle] (In1) -- (Sink);
        \draw[arrowstyle] (In2) -- (Sink);
        \draw[arrowstyle] (In3) -- (Sink);
    
        \node at ($(Sink) - (0,0.8)$) [below] {Objective: $\sum_{i} f_{(i,n)}$};
    
    \end{tikzpicture}
    \label{fig:sink}
    }
    \caption{Different node types in \sysname's DSL.}
    \label{fig:combined_nodes}
\end{figure*}

\newpage
\appendix
\noindent

\section{Formalizing $\mathcal{X}$Plain's DSL}
\label{sec:proof}

We prove that we can model any linear optimization in $\mathcal{X}$Plain.

\subsection{$\mathcal{X}$Plain's node description}
\parab{\textsc{Preliminaries}}. Our network-flow-based DSL is a directed graph where we denote the set of nodes with $\mathcal{N}$ and the set of directed edges as $\mathcal{E}$. We treat each edge $(i, j) \in \mathcal{E}$ as a variable with a non-negative flow value $f_{(i,j)} \geq 0$. We impose constraints on these flow variables as needed. We define incoming edges to node $n \in \mathcal{N}$ as those edges which are directed towards $n$ (i.e., $(i,n) \in \mathcal{E}$). Outgoing edges are those exiting $n$. The incoming (outgoing) traffic to a node is the sum of all flow that arrives at that node from all the incoming (outgoing) edges.

\noindent We have the following node behaviors:

\vspace{.5mm}
\parab{\textsc{split nodes} ($\mathcal{N}_{split}$)} split the incoming traffic between the outgoing edges (\autoref{fig:split}). They enforce the traditional flow conservation constraints: 

{\small
\begin{align*}
    & \sum_{\{i \in \mathcal{N}, (i,n) \in \mathcal{E}\}} f_{(i,n)} = \sum_{\{i \in \mathcal{N}, (n,i) \in \mathcal{E}\}} f_{(n,i)} && \forall n \in \mathcal{N}_{split} \\
\end{align*}
}






They can also optionally enforce (1) an upper bound on the traffic on an outgoing edge (capacity constraint) and (2) the traffic on an incoming edge to be constant. 
{\small
\begin{align*}
    & f_{(n, i)} \leq C_{(n, i)} \quad C_{(n, i)} \in \mathbb{R}^{+}, \forall i \in \{i \in \mathcal{N}, (n, i) \in \mathcal{E}\} && \forall n \in \mathcal{N}_{split} \\
    & f_{(i, n)} = d_{(i, n)} \quad d_{(i, n)} \in \mathbb{R}_{\geq 0}, \forall i \in \{i \in \mathcal{N}, (i, n) \in \mathcal{E}\} && \forall n \in \mathcal{N}_{split}
\end{align*}
}

\vspace{.5mm}
\parab{\textsc{pick nodes} ($\mathcal{N}_{pick}$)} satisfy flow conservation \emph{but} only allow one of the outgoing edges to carry traffic (\autoref{fig:pick}):

{\small
	\begin{align*}
	& \sum_{\{i \in \mathcal{N}, (i,n) \in \mathcal{E}\}} f_{(i,n)} = \sum_{\{i \in \mathcal{N}, (n,i) \in \mathcal{E}\}} f_{(n,i)} && \forall n \in \mathcal{N}_{pick} \\
	& \sum_{\{i \in \mathcal{N}, (n,i) \in \mathcal{E}\}}  \mathds{1}[f_{(n, i)} > 0] = 1 && \forall n \in \mathcal{N}_{pick}
	\end{align*}
}
where $\mathds{1}[x > 0]$ is an indicator function (=$1$ if $x > 0$, otherwise = 0). 







\vspace{.5mm}
\parab{\textsc{multiply nodes} ($\mathcal{N}_{mult}$)} only have one incoming and one outgoing link. They multiply the incoming traffic by a constant $C \in \mathbb{R}^{+}$ before sending it out (\autoref{fig:mult}). They only satisfy flow conservation when $C=1$.

{\small
	\begin{align*}
		f_{(n,i)} = C f_{(j,n)} && \forall (i,j) \in \{(i, j) \mid i, j \in \mathcal{N}, (n, i), (j,n) \in \mathcal{E}\}~\forall n \in \mathcal{N}_{mult}
	\end{align*}
}
    
    
    
    
    

\vspace{.5mm}
\parab{\textsc{all equal nodes} ($\mathcal{N}_{allEq}$)} require all the incoming and outgoing edges to carry the same amount of traffic (\autoref{fig:all_equal}):

{\small
	\begin{align*}
		f_{(n,i)} = f_{(j,n)} && \forall (i,j) \in \{(i, j) \mid i, j \in \mathcal{N}, (n, i), (j,n) \in \mathcal{E}\}~\forall n \in \mathcal{N}_{allEq}
	\end{align*}
}



    




\vspace{.5mm}

To make it simpler to encode a heuristic in the DSL, we also add the following node types to our DSL:

\vspace{.5mm}
\parab{\textsc{copy nodes} ($\mathcal{N}_{copy}$)} copy the total incoming flow into each outgoing edge (\autoref{fig:copy}):

{\small
	\begin{align*}
	f_{(n,j)} = \sum_{\{i \in \mathcal{N}, (i,n) \in \mathcal{E}\}} f_{(i,n)} && \forall j \in \{j \mid j \in \mathcal{N}, (n, j) \in \mathcal{E}\}~\forall n \in \mathcal{N}_{copy} \\
	\end{align*}
}

We can recreate this node's behavior if we combine split nodes and equal nodes (\autoref{fig:split_equal_copy}). However, using a copy node directly is more intuitive and straightforward for users, and we include it in our DSL for that reason. 

    
    
        
    
    
    

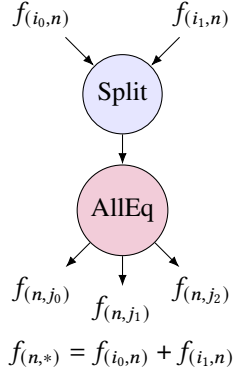
\begin{figure}[ht]
    \centering
\begin{tikzpicture}[node distance=1.5cm and 2cm, auto, scale=1, every node/.style={transform shape}]
    \tikzset{
        splitstyle/.style={circle, draw=black, fill=blue!10, text centered, minimum height=3em, minimum width=3em},
        arrowstyle/.style={-{Latex[length=1.5mm]}, draw=black},
        allEqstyle/.style={circle, draw=black, fill=purple!20, text centered, minimum height=3em, minimum width=3em},
    }

    \node[splitstyle] (Split) {Split};
    \node[allEqstyle, below=0.4cm of Split] (AllEq) {AllEq};
    \node[above left of=Split] (In1) {$f_{(i_0,n)}$};
    \node[above right of=Split] (In2) {$f_{(i_1,n)}$};

    \draw[arrowstyle] (In1) -- (Split);
    \draw[arrowstyle] (In2) -- (Split);

    \node[below left of=AllEq] (Out1) {$f_{(n,j_0)}$};
    \node[below right of=AllEq] (Out2) {$f_{(n,j_2)}$};
    \node[below=0.4cm of AllEq] (Out3) {$f_{(n,j_1)}$};

    \draw[arrowstyle] (AllEq) -- (Out1);
    \draw[arrowstyle] (AllEq) -- (Out2);
    \draw[arrowstyle] (AllEq) -- (Out3);
    
    \draw[arrowstyle] (Split) -- (AllEq);

    \node at ($(AllEq) - (0,1.6)$) [below] {$f_{(n,*)} = f_{(i_0,n)} + f_{(i_1,n)}$};
\end{tikzpicture}
\caption{Recreating \textsc{copy node} with \textsc{split node} and \textsc{all equal node}}
    \label{fig:split_equal_copy}
\end{figure}

We use source and sink nodes to define the objective:

\vspace{.5mm}
\parab{\textsc{source nodes} ($\mathcal{N}_{source}$)} are special cases of split or pick nodes that represent the inputs to the problem. For example, \autoref{fig:dsl_DP} illustrates the input traffic demand modeled as source nodes that enforce split node behavior (\sourcesplit). Also, \autoref{fig:dsl_vbp} shows the input ball sizes as source nodes with pick node behavior (\sourcepick, each ball can only be placed in one bin).

\vspace{.5mm}
\parab{\textsc{sink node} ($\mathcal{N}_{sink}$)} is a specific node that (1) only has incoming edges and (2) measures the performance of the problem as the total incoming traffic through these edges (\autoref{fig:sink}). When the DSL represents an optimization problem, the sink node is designated as the objective, and the compiler translates the value of the sink node into the optimization objective.

    
    
    
    
    

\subsection{$\mathcal{X}$Plain can model any linear optimization}
\vspace{.5mm}
\begin{theorem}
    We can model any linear optimization (linear programming or mixed integer linear programming) as a flow network using the six node behaviors ($\mathcal{N}_{split}$, $\mathcal{N}_{pick}$, $\mathcal{N}_{mult}$, $\mathcal{N}_{allEq}$, and $\mathcal{N}_{sink}$)
\end{theorem}

\begin{proof} An optimization problem maximizes (or minimizes) an objective subject to inputs that fall within a feasible space that the optimization constraints characterize.  We can express a linear optimization problem as (linear programming or mixed integer linear programming):

{
\small
    \begin{align*}
        \max_{{\bf x}, {\bf y}}~~~{\bf c}_{\bf x}^\intercal {\bf x} + {\bf c}_{\bf y}^\intercal {\bf y} \\
        {\bf A}_{\bf x} {\bf x} + {\bf A}_{\bf y} {\bf y} \leq {\bf b} \\
        {\bf x} \geq {\bf 0} \\
        {\bf y} \in \{0, 1\}^{|{\bf y}|} 
    \end{align*}
}

To show that our DSL is complete, we need to show that we can capture both the feasible space and the objective correctly through our flow model for every possible linear optimization. 

We first present a general algorithm to express the feasible space of any given linear optimization as a flow model and prove it is correct. Next, we show how we can use the same algorithm to express any linear objective.

\vspace{1.5mm}
\noindent {\bf How to represent the feasible space with a flow model.} We can express the feasible space of any linear optimization as:
{
\small
    \begin{align}
        {\bf A}_{\bf x} {\bf x} + {\bf A}_{\bf y} {\bf y} \leq {\bf b} \label{eqConstraint} \\
        {\bf x} \geq {\bf 0} \\
        {\bf y} \in \{0, 1\}^{|{\bf y}|} 
    \end{align}
}
where we denote matrices and vectors in bold. {\bf x} and {\bf y} are vectors of continuous and binary variables of size $|\mathbf{x}| \times 1$ and $|\mathbf{y}| \times 1$ , respectively. ${\bf b}$ is a constant vector of size $|\mathbf{b}| \times 1$ . ${\bf A}_{\bf x}$ and ${\bf A}_{\bf y}$ are constant matrices of sizes $\bf |b| \times \bf |x|$ and $\bf |b| \times \bf |y|$ respectively. Note that we can enforce an equality constraint as two inequality constraints (\autoref{eqConstraint}), and represent any integer variable as the sum of multiple binary variables. We map the variables to flows in our model.

We need to transform the above optimization before we can model it with our node behaviors:

\vspace{2mm}
\noindent $\blacktriangleright$ {\it Transformation 1.} The matrices ${\bf A}_{\bf x}$ and ${\bf A}_{y}$, and the vector ${\bf b}$ may contain negative entries. This conflicts with the non-negativity requirement of the flows in our flow model. To address this, we decompose these matrices and vector into their positive and negative components:
{
\small
    \begin{align*}
        {\bf A}_{\bf x} = {\bf A}^{+}_{\bf x} - {\bf A}^{-}_{\bf x},~~~ 
        {\bf A}_{\bf y} = {\bf A}^{+}_{\bf y} - {\bf A}^{-}_{\bf y},~~~
        {\bf b} = {\bf b}^{+} - {\bf b}^{-}
    \end{align*}
}
where all the elements in $\mathbf{A^{+}_x} = [a^{(+, \mathbf{x})}_{ij}]$ and $\mathbf{A^{-}_x} = [a^{(-, \mathbf{x})}_{ij}]$ are non-negative such that \textbf{at most one} of $a^{(+, \mathbf{x})}_{ij}$ or $a^{(-, \mathbf{x})}_{ij}$ is non-zero for every $i \in \mathbb{Z}_{[0, |\mathbf{b}|)}$ and $j \in \mathbb{Z}_{[0, |\mathbf{x}|)}$. Note that $\mathbb{Z}_{[0, m)} = \{0, \dots, m-1\}$. Same holds for both (1) $\bf A^{+}_y$ and $\bf A^{-}_y$, and (2) $\mathbf{b^{+}} = [b^+_i]$ and $\mathbf{b^{-}} = [b^-_i]$ over every $i$. All matrices have the same size as their originating matrix. After substituting these decompositions into~\autoref{eqConstraint}, we have:

{
\small
    \begin{align}
    {\bf A}^{+}_{\bf x} {\bf x} + {\bf A}^{+}_{\bf y} {\bf y} + {\bf b}^{-} \leq  {\bf A}^{-}_{\bf x} {\bf x} + {\bf A}^{-}_{\bf y}{\bf y} + {\bf b}^{+} \label{eqConstraintV2}
    \end{align}
}

\vspace{2mm}
\noindent $\blacktriangleright$ {\it Transformation 2.} \autoref{eqConstraintV2} and \textsc{split node}s qualitatively represent similar behaviors. \textsc{split node}s split the incoming traffic across outgoing edges and ensure the traffic on each edge does not exceed the capacity constraints. Ideally, we can enforce the \autoref{eqConstraintV2} constraints using a \textsc{split node}s and as a flow conservation a constraint:
{
\small
    \begin{align}
    & {\bf A}^{+}_{\bf x} {\bf x} + {\bf A}^{+}_{\bf y} {\bf y} + {\bf b}^{-} + {\bf f} =  {\bf A}^{-}_{\bf x} {\bf x} + {\bf A}^{-}_{\bf y}{\bf y} + {\bf b}^{+}  && \text{(Flow conservation)} \noindent \nonumber \\ 
    & {\bf 0} \leq {\bf f} && \text{(Flow constraint)} \label{eqConstraintV4}
    \end{align}
}

\vspace{0.2mm}
The problem is that \autoref{eqConstraintV2} also involves coefficients associated with each variable (${\bf A}$), while \textsc{split node}s do not accept weights. We address this by replacing each term (coefficient multiplied by a variable) in each of the \autoref{eqConstraintV2} constraints with an auxiliary variable:
\vspace{1mm}

{
\small
    \begin{align}
        & u^{+}_{ij} = a^{(+, \bf x)}_{ij} x_j, \; u^{-}_{ij} = 0 \quad \text{if}~a^{(+, \bf x)}_{ij} \geq 0 \quad  \forall{i} \in \mathbb{Z}_{[0, |\mathbf{b}|)}, \forall{j} \in \mathbb{Z}_{[0, |\mathbf{x}|)} \nonumber \\
        & u^{-}_{ij} = a^{(-, \bf x)}_{ij} x_j, \; u^{+}_{ij} = 0 \quad \text{if}~a^{(-, \bf x)}_{ij} > 0 \quad \forall{i} \in \mathbb{Z}_{[0, |\mathbf{b}|)}, \forall{j} \in \mathbb{Z}_{[0, |\mathbf{x}|)} \nonumber \\
        & v^{+}_{ij} = a^{(+, \bf y)}_{ij} y_j, \; v^{-}_{ij} = 0 \quad \text{if}~a^{(+, \bf y)}_{ij} \geq 0 \quad \forall{i} \in \mathbb{Z}_{[0, |\mathbf{b}|)}, \forall{j} \in \mathbb{Z}_{[0, |\mathbf{y}|)} \nonumber \\ 
        & v^{-}_{ij} = a^{(-, \bf y)}_{ij} y_j, \; v^{+}_{ij} = 0  \quad \text{if}~a^{(-, \bf y)}_{ij} > 0 \quad \forall{i} \in \mathbb{Z}_{[0, |\mathbf{b}|)}, \forall{j} \in \mathbb{Z}_{[0, |\mathbf{y}|)} \label{eq:AuxVarY}
    \end{align}
}

\vspace{1.5mm}

We define $\mathbf{U}^{+} = [u^{+}_{ij}]$, $\mathbf{U}^{-} = [u^{-}_{ij}]$, $\mathbf{V}^{+} = [v^{+}_{ij}]$, and $\mathbf{V}^{-} = [v^{-}_{ij}]$. We can then express \autoref{eqConstraintV4} in terms of these auxiliary variables:

{
\small
    \begin{align}
        & {\bf U}^{+}{\bf d}_{\bf x} + {\bf V}^{+}{\bf d}_{\bf y} + {\bf b}^{-} + {\bf f} = {\bf U}^{-}{\bf d}_{\bf x} + {\bf V}^{-}{\bf d}_{\bf y} + {\bf b}^{+},~~ 0 \leq {\bf f} \nonumber
    \end{align}
}

where ${\bf d}_{\bf x}$ and ${\bf d}_{\bf y}$ are vectors with all elements equal to 1 and sizes of $|\mathbf{x}| \times 1$ and $|\mathbf{y}| \times 1$ respectively. This is because each of the auxiliary variables $u_{ij}$ or $v_{ij}$ appear in exactly one inequality constraint.

\vspace{2mm}
\noindent $\blacktriangleright$ {\it Transformation 3.} We encounter a problem to enforce the constraints in \autoref{eq:AuxVarY} using \textsc{multiply node} for $u_{ij}$ and $v_{ij}$: each \textsc{multiply node} has only one input and one output edge. Each edge also corresponds to one variable. This means each variable can appear in at most two constraints, corresponding to the two nodes at the two ends of the edge. However, the variables in \autoref{eq:AuxVarY} appear more than twice (for example, $x_j$ can appear up to $|\mathbf{b}|$ times.)

We address this by introducing additional variables and constraints:

{
\small
    \begin{align}
        & u^{+}_{ij} = a^{(+, \bf x)}_{ij} x^+_{ij},~~~u^{-}_{ij} = a^{(-, \bf x)}_{ij} x^-_{ij} & \forall{i} \in \mathbb{Z}_{[0, |\mathbf{b}|)}, \forall{j} \in \mathbb{Z}_{[0, |\mathbf{x}|)}  \nonumber\\
        & v^{+}_{ij} = a^{(+, \bf y)}_{ij} y^+_{ij},~~~v^{-}_{ij} = a^{(-, \bf y)}_{ij} y^-_{ij} & \forall{i} \in \mathbb{Z}_{[0, |\mathbf{b}|)}, \forall{j} \in \mathbb{Z}_{[0, |\mathbf{y}|)}  \nonumber\\
        & x^+_{ij} = x^-_{ij} = x_j & \forall{i} \in \mathbb{Z}_{[0, |\mathbf{b}|)}, \forall{j} \in \mathbb{Z}_{[0, |\mathbf{x}|)} \nonumber \\
        & y^+_{ij} = y^-_{ij} = y_j & \forall{i} \in \mathbb{Z}_{[0, |\mathbf{b}|)}, \forall{j} \in \mathbb{Z}_{[0, |\mathbf{y}|)}  \nonumber
    \end{align}
}
With these modifications, each variable $x^+_{ij}$ and $x^-_{ij}$ appears in exactly two constraints (same for y).
\vspace{0.5mm}

The final resulting optimization after all the transformations is:

{
\small
    \begin{align}
        &&&  {\bf U}^{+}{\bf d}_{\bf x} + {\bf V}^{+}{\bf d}_{\bf y} + {\bf b}^{-} + {\bf f} = {\bf U}^{-}{\bf d}_{\bf x} + {\bf V}^{-}{\bf d}_{\bf y}  + {\bf b}^{+}, ~~ \bf 0 \leq {\bf f}  \label{eq:split} \\
        &&& u^{+}_{ij} = a^{(+, \bf x)}_{ij} x^+_{ij} ~~~~~~~~~~~~~~~~~~~~~~~~~~~~~~~~~~\forall {i}~~\forall{j} \label{eq:mult1}\\
        &&&  x^-_{ij} = \frac{1}{a^{(-, \bf x)}_{ij}} u^{-}_{ij} ~~~~~ \text{if}~~a^{(-, \bf x)}_{ij} > 0 ~~~~~~\forall {i}~~\forall{j} \label{eq:mult2}\\
        &&& v^{+}_{ij} = a^{(+, \bf y)}_{ij} y^+_{ij} ~~~~~~~~~~~~~~~~~~~~~~~~~~~~~~~~~~ \forall {i}~~\forall{j} \label{eq:mult3} \\
        &&& y^-_{ij} = \frac{1}{a^{(-, \bf y)}_{ij} } v^{-}_{ij} ~~~~~ \text{if}~~a^{(-, \bf y)}_{ij} > 0 ~~~~~~\forall {i}~~\forall{j} \label{eq:mult4}\\
        &&& x^+_{ij} = x^-_{ij} = x_j ~~~~~~~~~~~~~~~~~~~~~~~~~~~~~~~~~~~ \forall {i}~~\forall{j} \label{eq:alleq1}\\
        &&& y^+_{ij} = y^-_{ij} = y_j ~~~~~~~~~~~~~~~~~~~~~~~~~~~~~~~~~~~ \forall {i}~~\forall{j} \label{eq:alleq2}\\
        &&& {\bf x} \geq \bf 0 \label{eq:pos}\\
        &&& {\bf y} \in \{0, 1\}^{|\bf y|}  \label{eq:bin}
    \end{align}
}

\noindent where for each of the equations above, notation $\forall {i}~~\forall{j}$ means all the possible $i$ and $j$ values should be considered according to the specific constraints or conditions given for each equation.

\vspace{2mm}
\noindent $\blacktriangleright$ {\it Constructing the flow model.}
We can encode the above constraints using a flow model. We first create one edge per variable and then enforce each constraint using one node: 

\begin{itemize}
    \item[{\bf (S1)}] We encode \autoref{eq:split} using \textsc{split node}s. We will have a node for each possible $i$. The inputs to each node are (1) one edge per variable on the left-hand side of the constraint ($\bf U^{+}$ and $\bf V^{+}$), (2) one edge with a constant rate $\bf b^-$, and (3) one additional edge associated with $\bf f$. The outputs are (1) one edge per variable on the right-hand side of the constraint ($\bf U^{-}$ and $\bf V^{-}$)  and (2) one additional edge with constant rate $\bf b^+$. ~\autoref{fig:step1} shows how this encoding is done. 

    \begin{figure}[h]
        \centering
    \begin{tikzpicture}[node distance=1.5cm and 2cm, auto, scale=1, every node/.style={transform shape}]
        \tikzset{
            splitstyle/.style={circle, draw=black, fill=blue!10, text centered, minimum height=3em, minimum width=3em},
            arrowstyle/.style={-{Latex[length=1.5mm]}, draw=black},
        }
    
        \node[splitstyle] (Split) {Split$(i)$};
    
        \node[above left of=Split] (In1) {$\forall_{j} ~ u^{+}_{ij}$};
        \node[above right of=Split] (In2) {$b^{-}_{i}$};
        \node[above=0.4cm of Split] (In3) {$\forall_{j}  ~ v^{+}_{ij}$};
        \node[above left=-0.4cm and 0.7cm of Split] (In4) {$f_i$};

        \node[below left of=Split] (Out1) {$\forall_{j} ~ u^{-}_{ij}$};
        \node[below right of=Split] (Out2) {$b^{+}_{i}$};
        \node[below=0.4cm of Split] (Out3) {$\forall_{j} v^{-}_{ij}$};
    
        \draw[arrowstyle] (In1) -- (Split);
        \draw[arrowstyle] (In2) -- (Split);
        \draw[arrowstyle] (In3) -- (Split);
        \draw[arrowstyle] (In4) -- (Split);
        \draw[arrowstyle] (Split) -- (Out1);
        \draw[arrowstyle] (Split) -- (Out2);
        \draw[arrowstyle] (Split) -- (Out3);
        \node at ($(Split) - (0,1.9)$) [below] {$\sum_{j} [u^{+}_{ij} + v^{+}_{ij}] + b^{-}_{i} + f_i = \sum_{j} [u^{-}_{ij} + v^{-}_{ij}] + b^{+}_{i}$};
    \end{tikzpicture}
    \caption{Step 1 of the encoding: \textsc{split node} for $i$. There will be a \textsc{split node} for each possible $i \in \mathbb{Z}_{[0, |\mathbf{b}|)}$. If a variable is 0, we do not need to assign it to the node. There are at most $|\mathbf{x}|$ arrows present for $u^{+}_{ij}$ and $u^{-}_{ij}$ since at most one of $a^{(-, \bf x)}_{ij}$ or $a^{(+, \bf x)}_{ij}$ is non-zero. Similarly, there are at most $|\mathbf{y}|$ arrows present for $y^{+}_{ij}$ and $y^{-}_{ij}$.}
        \label{fig:step1}
    \end{figure}

    \item[{\bf (S2)}] We express \autoref{eq:mult1}~--~\ref{eq:mult4} using \textsc{multiply node}s. The $\bf U^{-}$ edges originate from \textsc{split node}s to these \textsc{multiply node}s while $\bf U^{+}$ edges are in the opposite direction. So, the node that models \autoref{eq:mult1} has $x^+_{ij}$ as its input edge and $u^{+}_{ij}$ as its output edge. Conversely, the input edge is $u^{-}_{ij}$ and the output edge is $x^{-}_{ij}$ for \autoref{eq:mult2} (same holds for $y$ and $v$).~\autoref{fig:step2} shows this step.

    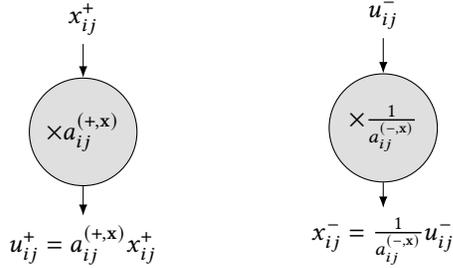
\begin{figure}[h]
        \centering
        \begin{subfigure}[b]{0.4\columnwidth}
            \centering
            \begin{tikzpicture}[node distance=1.5cm and 2cm, auto, scale=1, every node/.style={transform shape}]
                \tikzset{
                    multstyle/.style={circle, draw=black, fill=gray!25, text centered, minimum height=3em, minimum width=3em},
                    arrowstyle/.style={-{Latex[length=1.5mm]}, draw=black},
                }
            
                \node[multstyle] (Mult) {$\times a^{(+, \bf x)}_{ij}$};
            
                \node[above of=Mult] (In) {$x^+_{ij}$};
            
                \node[below of=Mult] (Out) {$u^{+}_{ij} = a^{(+, \bf x)}_{ij} x^+_{ij}$};
            
                \draw[arrowstyle] (In) -- (Mult);
                \draw[arrowstyle] (Mult) -- (Out);
            
            \end{tikzpicture}
            \label{fig:step2_mult1}
        \end{subfigure}
        \hspace{0.05\columnwidth} 
        \begin{subfigure}[b]{0.4\columnwidth}
            \centering
            \begin{tikzpicture}[node distance=1.5cm and 2cm, auto, scale=1, every node/.style={transform shape}]
                \tikzset{
                    multstyle/.style={circle, draw=black, fill=gray!25, text centered, minimum height=3em, minimum width=3em},
                    arrowstyle/.style={-{Latex[length=1.5mm]}, draw=black},
                }
            
                \node[multstyle] (Mult) {$\times \frac{1}{a^{(-, \bf x)}_{ij}}$};
            
                \node[above of=Mult] (In) {$u^{-}_{ij}$};
            
                \node[below of=Mult] (Out) {$x^-_{ij} = \frac{1}{a^{(-, \bf x)}_{ij}} u^{-}_{ij}$};
            
                \draw[arrowstyle] (In) -- (Mult);
                \draw[arrowstyle] (Mult) -- (Out);
            
            \end{tikzpicture}
            \label{fig:step2_mult2}
        \end{subfigure}
        \caption{Step 2 of the encoding. There will be a \textsc{multiply node} for each possible $i$ and $j$. At most of these two \textsc{multiply node}s will be needed since at most one of $a^{(-, \bf x)}_{ij}$ or $a^{(+, \bf x)}_{ij}$ is non-zero.}
        \label{fig:step2}
    \end{figure}

    \item[{\bf (S3)}] We model \autoref{eq:alleq1}~--~\ref{eq:alleq2} using \textsc{all equal node}s.
    Note that for a fixed $i$ and $j$, since at most one of $a^{(-, \bf x)}_{ij}$ and $a^{(+, \bf x)}_{ij}$ is non-zero, at most of the equations in \autoref{eq:mult1} and \autoref{eq:mult2} are needed for that $i$ and $j$ (same holds for \autoref{eq:mult3} and \autoref{eq:mult4}). Consequently, at most of $x^+_{ij}$ and $x^-_{ij}$ is needed in \autoref{eq:alleq1} (same holds for $y^+_{ij}$ and $y^-_{ij}$ in \autoref{eq:alleq2}). The $x_{j}$ and $x^{-}_{ij}$s are input edges and $x^{+}_{ij}$s are the output edges (same for $y$).~\autoref{fig:step3} illustrates this step.

    \begin{figure}[h]
        \centering
        \begin{tikzpicture}[node distance=1.5cm and 2cm, auto, scale=1, every node/.style={transform shape}]
            \tikzset{
                multstyle/.style={circle, draw=black, fill=purple!20, text centered, minimum height=3em, minimum width=3em},
                arrowstyle/.style={-{Latex[length=1.5mm]}, draw=black},
            }
        
            \node[multstyle] (AllEq) {AllEq$(j)$};
        
            \node[above left of=AllEq] (In1) {$x_j$};
            \node[above right of=AllEq] (In2) {$\forall_{i}~{x^{-}_{ij}}$};
            \node[below of=AllEq] (Out) {$\forall_{i}~{x^+_{ij}}$};
        
            \draw[arrowstyle] (In1) -- (AllEq);
            \draw[arrowstyle] (In2) -- (AllEq);
            \draw[arrowstyle] (AllEq) -- (Out);
        
        \end{tikzpicture}
    \caption{Step 3 of the encoding. There will be a \textsc{all equal node} for each possible $j \in \mathbb{Z}_{[0, |\mathbf{x}|)}$.}

    \label{fig:step3}
    \end{figure}
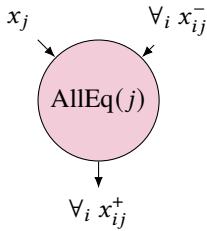

    \item[{\bf (S4)}] The input variables are the variables in $\bf x$ and $\bf y$. We represent binary variables in \autoref{eq:bin} using \textsc{pick nodes}. It has one incoming edge with a constant rate of 1 and two outgoing edges. One of the outputs corresponds to the binary variable. If the node selects that specific edge to carry the flow, the binary variable is 1. Otherwise, it is 0. \autoref{eq:pos} is inherently satisfied as flows are all non-negative.
    
\end{itemize}

    
    

    


This flow model provably captures the optimization's feasible space as there is a one-to-one correspondence between the constraints in the optimization and the constraints enforced by the nodes.

\noindent {\bf How to capture the optimization objective.} We can express the objective of any linear optimization as $~~\max_{{\bf x}, {\bf y}}~~~{\bf c}_{\bf x}^\intercal {\bf x} + {\bf c}_{\bf y}^\intercal {\bf y}$ where ${\bf c}_{\bf x}$ and  ${\bf c}_{\bf y}$ are constant vectors. We can reformulate and add a constraint that enforces $p = {\bf c}_{\bf x}^\intercal {\bf x} + {\bf c}_{\bf y}^\intercal {\bf y}$, so the objective of the optimization changes to maximizing $p$. Then, we can use similar transformations, as we explained before, to capture this constraint within the flow model. We add a sink node that has one incoming edge $p$. This way, we can express any linear optimization objective with our model.

\end{proof}

\clearpage
\fi

\end{document}